\documentclass[sigconf]{acmart}
\usepackage{bm} 
\usepackage{algorithm}
\usepackage{algorithmic}
\usepackage{amsthm}
\usepackage{wrapfig}
\usepackage{enumitem}
\usepackage{microtype}
\usepackage{multirow}

\newtheorem{remark}{Remark}

\AtBeginDocument{%
  }

\copyrightyear{2026}
\acmYear{2026}
\setcopyright{cc}
\setcctype{by}
\acmConference[KDD '26]{Proceedings of the 32nd ACM SIGKDD Conference on Knowledge Discovery and Data Mining V.2}{August 09--13, 2026}{Jeju Island, Republic of Korea}
\acmBooktitle{Proceedings of the 32nd ACM SIGKDD Conference on Knowledge Discovery and Data Mining V.2 (KDD '26), August 09--13, 2026, Jeju Island, Republic of Korea}
\acmDOI{10.1145/3770855.3818906}
\acmISBN{979-8-4007-2259-2/2026/08}
\acmSubmissionID{v2ais267}

\begin{document}

\title{NPSolver: Neural Poisson Solver with Iterative Physics Supervision}


\author{Bocheng Zeng}
\authornote{Equal contributions.}
\affiliation{%
  \department{Gaoling School of Artificial Intelligence}
  \institution{Renmin University of China}
  \city{Beijing}
  \country{China}}
\email{zengbocheng@ruc.edu.cn}

\author{Rui Zhang}
\authornotemark[1]
\affiliation{%
  \department{Gaoling School of Artificial Intelligence}
  \institution{Renmin University of China}
  \city{Beijing}
  \country{China}}
\email{rayzhang@ruc.edu.cn}

\author{Runze Mao}
\affiliation{%
  \department{School of Mechanics and Engineering Science}
  \institution{Peking University}
  \city{Beijing}
  \country{China}}
\email{maorz1998@stu.pku.edu.cn}

\author{Mengtao Yan}
\affiliation{%
  \department{Gaoling School of Artificial Intelligence}
  \institution{Renmin University of China}
  \city{Beijing}
  \country{China}}
\email{mengtaoyan@ruc.edu.cn}

\author{Xuan Bai}
\affiliation{%
  \institution{AI for Science Institute}
  \city{Beijing}
  \country{China}}
\email{xuan.bai@pku.edu.cn}

\author{Yang Liu}
\affiliation{%
  \department{School of Engineering Science}
  \institution{University of Chinese Academy of Sciences}
  \city{Beijing}
  \country{China}}
\email{liuyang22@ucas.ac.cn}

\author{Zhi X. Chen}
\affiliation{%
  \department{School of Mechanics and Engineering Science}
  \institution{Peking University}
  \city{Beijing}
  \country{China}}
\email{chenzhi@pku.edu.cn}

\author{Hao Sun}
\authornote{Corresponding author.}
\affiliation{%
  \department{Gaoling School of Artificial Intelligence}
  \institution{Renmin University of China}
  \city{Beijing}
  \country{China}}
\email{haosun@ruc.edu.cn}

\renewcommand{\shortauthors}{Bocheng Zeng et al.}

\begin{abstract}
Efficiently solving Poisson equations on complex, irregular domains remains a fundamental challenge in scientific computing, as classical iterative solvers often suffer from prohibitive runtime due to ill-conditioned systems. While neural operators offer a fast alternative, they typically rely on large-scale labeled datasets or struggle with unstable training dynamics when using physics-informed residual losses. We propose \textsc{NPSolver}, a neural Poisson solver trained without solution labels via iterative physics supervision. Instead of relying on fully converged numerical solutions or raw PDE residuals, \textsc{NPSolver} utilizes a small number of preconditioned conjugate gradient (PCG) steps to refine its own predictions, providing a more stable and well-scaled training signal. Theoretical analysis confirms that this iterative supervision serves as a well-conditioned error proxy and that a stop-gradient design is essential for optimization stability. To better capture boundary-driven features under mixed boundary conditions, we further introduce the Boundary-Aware Transolver (\textsc{BA-Transolver}) architecture that explicitly separates interior and boundary tokenization. Extensive evaluations on 2D and 3D irregular geometries demonstrate that \textsc{NPSolver} outperforms both physics-informed and data-driven baselines. Furthermore, a downstream thermal control task highlights the model's capability for conducting efficient and reliable gradient-based boundary control. We will release our codes and data at https://github.com/intell-sci-comput/NPSolver.

\end{abstract}





\maketitle

\section{Introduction}
Efficiently solving Poisson equations on complex and irregular domains remains a major computational bottleneck in many scientific and engineering pipelines. As a cornerstone PDE, the Poisson operator underlies a wide range of physical processes, including pressure projection in incompressible flows~\cite{chorin1968}, steady-state heat conduction~\cite{patankar1980numerical}, and
electrostatics~\cite{jackson1998classical}. Classical methods discretize the PDE using numerical schemes such as finite volumes~\cite{Moukalled2016}, finite elements~\cite{Zienkiewicz2013}, or finite differences~\cite{LeVeque2007}, resulting in a large, sparse linear system solved by iterative methods, e.g., preconditioned conjugate gradient (PCG)~\cite{cg1952, saad2003} or multigrid~\cite{briggs2000, TrotMult2001}. While these solvers are robust, their runtime is dominated by the iteration count, which increases sharply for irregular geometries and challenging boundary conditions (BCs) that induce ill-conditioned systems.

To accelerate classical solvers, learning-based approaches have emerged as a promising alternative, using neural networks as fast surrogate models. This paradigm includes operator-learning methods~\cite{fno2021, deeponet2021, ffno2023, uno2023, Wan_Wang_Mi_Zhang_Sun_2026} trained on paired input--solution data, as well as graph- and transformer-based surrogates~\cite{gino2023, wu2024Transolver, luo2025transolverplus} that better handle irregular meshes and unstructured grids. Despite their fast inference, most such methods rely on large collections of ground-truth solutions generated by expensive numerical solvers, which can be prohibitive for complex geometries and diverse BCs. 

To reduce the dependence on labeled data, Physics-Informed Neural Networks (PINNs)~\cite{raissi2018dhpm, xpinn2020, gpinn2022} replace data supervision with physics losses, e.g., PDE residuals and boundary constraints. However, optimizing physics-informed objectives is often unstable and costly in practice. On the one hand, PINNs can struggle with stiff operators and ill-conditioned regimes, leading to slow convergence and high sensitivity to hyperparameters~\cite{krishnapriyan2021characterizing,wang2022and}. On the other hand, PINNs typically require per-instance optimization, necessitating retraining when the domain geometry or BCs change~\cite{grossmann2024can}. Recent physics-informed operator learning methods~\cite{pino, wang2021pideeponet, zhang2025omnifluid} further extend this idea by using physics supervision to train neural operators, thereby avoiding the per-instance retraining like PINNs. However, these methods still rely on residual-based objectives (strong-form residuals via automatic differentiation~\cite{wang2021pideeponet} or discretized PDE residuals~\cite{pino}), where gradients can be poorly scaled in stiff or ill-conditioned regimes, leading to slow convergence and unstable training dynamics. Moreover, many existing approaches are developed under regular-grid backbones and simple geometry, which limits their ability to generalize across BCs and complex, irregular domains.

We propose \textsc{NPSolver}, a neural Poisson solver trained \emph{without} solution labels via iterative physics supervision. During training, \textsc{NPSolver} applies a small number ($K$) of PCG steps to the network prediction and uses the PCG-updated iterate as the supervision target. This bypasses the need for fully converged numerical labels that typically require $T$ iterations, where $K\ll T$. Compared with physics-informed methods that directly minimize PDE residuals, \textsc{NPSolver} adopts iterative physics supervision, which provides a more stable and well-scaled training signal. We further provide theoretical guarantees showing that the induced self-consistency residual is a well-conditioned error proxy and that the stop-gradient design yields more favorable optimization dynamics. For architecture design, to improve generalization across domain geometries and BCs, we introduce a Boundary-Aware Transolver variant (\textsc{BA-Transolver}) that can handle various boundary-value problems on irregular meshes. BA-Transolver separately tokenizes interior and boundary nodes and performs attention over the combined token set to explicitly model complex boundary--interior interactions. 

We validate \textsc{NPSolver} on three challenging settings: (i) 2D Poisson problems on irregular domains with varying forcing fields and boundary regimes, (ii) 3D Poisson problems on a cube-with-cylindrical-hole family, and (iii) a downstream thermal control task on a perforated plate. We evaluate both \textit{in-distribution} accuracy and \textit{out-of-distribution} generalization under geometry shifts and boundary-condition changes. In the most challenging RandomBC regime, \textsc{NPSolver} achieves \textbf{8.17\%} relative $L_2$ error without any paired labels, outperforming the strongest supervised baseline trained with \textbf{7k} labeled samples (\textbf{10.17\%}). Moreover, it delivers a \textbf{15$\times$} speedup at matched accuracy compared with classical numerical methods. In summary, we make the following contributions:

1. We present \textsc{NPSolver}, a label-free neural Poisson solver built on \textsc{BA-Transolver}, a boundary-aware attention model designed to handle complex BCs. \textsc{NPSolver} is trained with an iterative physics supervision objective, avoiding the cost of generating fully converged solution pairs.

2. We provide theoretical guarantees for iterative physics supervision. The induced self-consistency residual is a well-conditioned error proxy, and the {stop-gradient} design leads to more stable and well-scaled training than directly minimizing PDE residuals.

3. We conduct comprehensive evaluations in 2D and 3D irregular domains, and a downstream control task, demonstrating that \textsc{NPSolver} offers a favorable error–cost trade-off and strong \textit{out-of-distribution} generalization for Poisson equations.

\section{Related Work}

{\textbf{Supervised neural surrogates.}} Given enough training data, a prominent research direction involves learning solution operators in a data-driven manner. Neural operator models, such as FNO~\cite{fno2021}, DeepONet~\cite{deeponet2021}, and their invariants~\cite{ffno2023, li2023gfno, LU2022114778}, learn direct mappings from problem inputs to solution fields and have demonstrated strong performance across benchmark PDE families. To handle irregular geometries and unstructured discretizations, graph-based and geometric models~\cite{pfaff_mgn, gino2023, li2025,zeng2025phympgn} operate directly on meshes or point clouds. Transformer-based surrogates~\cite{li2023trans, wu2024Transolver, luo2025transolverplus, han2022predicting} leverage attention mechanisms to model global interactions within the solution domain. Despite their inference speed, these supervised approaches typically require massive datasets of paired solution labels generated by high-fidelity numerical solvers. 

\noindent{\textbf{Physics-informed learning.}} Classical Physics-Informed Neural Networks (PINNs) and their variants~\cite{raissi_pinn, raissi2018dhpm, xpinn2020, gpinn2022} optimize networks by minimizing strong-form PDE residuals alongside boundary condition penalties. While these have shown promising results across various physical systems, a complementary family of approaches \cite{weinan2018, dgm2018, vpinn2019, hpvpinn2021} leverages weak forms or variational principles to improve training stability and reduce sensitivity to higher-order derivatives. Despite being label-free, these physics-informed methods are typically optimized per instance, i.e., requiring a separate optimization process for each new problem configuration. To develop reusable neural solvers, physics-informed operator methods combine operator learning with physics-based objectives~\cite{pino, wang2021pideeponet, li2024genfvgn,zhang2025omnifluid, zhang2025monte}, aiming for operator-level generalization. However, many current physics-informed operator frameworks still rely on regular-grid backbones or assume fixed geometries. This reliance hinders their direct application and limits their ability to generalize to irregular domains.

\noindent{\textbf{Learning-based iterative solvers.}} A distinct line of research aims to reduce the computational cost of iterative algorithms by learning specialized update rules or integrating neural predictors with classical solvers. LISTA~\cite{lista2010} learns a finite sequence of iterative updates by unrolling an optimization algorithm into neural layers, while implicit fixed-point models like DEQ~\cite{deq2019} define networks via equilibrium equations solved through root-finding during inference. For PDEs, \citeauthor{hsieh2018learning} learn neural modifications to iterative solvers that maintain convergence guarantees, and HINTS~\cite{zhang2024hints} blends neural operators with relaxation methods to refine solutions iteratively. In contrast, we run a small number of PCG steps only during training to generate physics self-supervision targets, producing predictions in a single forward pass during inference. 

\noindent{\textbf{Iterative physics supervision.}} Jia et al.~\cite{JIA2025114132} proposed an incomplete solve loss that generates detached supervision targets through truncated solver iterations initialized from network predictions. NPSolver was developed independently, and we became aware of their work only after completing our research. While sharing a closely related underlying idea, NPSolver employs PCG-based iterative physics supervision, provides theoretical analysis of PCG updates, and introduces BA-Transolver to handle irregular domains and varying boundary conditions.

\section{Methods}
\subsection{Problem Formulation}
\label{sec:problem}
We consider the Poisson equation in $d\in\{2,3\}$ spatial dimensions on a computational domain $\Omega \subset \mathbb{R}^d$ bounded by $\partial\Omega$. Given a forcing field $f:\Omega \rightarrow \mathbb{R}$, a Poisson solver aims to find a scalar field $u:\Omega \rightarrow \mathbb{R}$ such that
\begin{equation}
\begin{aligned}
\nabla^2 u(\bm {x}) &= f(\bm {x}), & \bm {x} \in \Omega, \\
\mathcal{B}(u)(\bm {x}) &= g(\bm {x}), & \bm {x} \in \partial\Omega
\end{aligned}
\end{equation}
where $\mathcal{B}$ denotes the boundary operator. In this work, we focus on Dirichlet and Neumann boundary conditions (BCs). Specifically, we prescribe $u(\bm {x}) = g_D(\bm {x}), \bm {x}\in  \partial\Omega_D$ and $
\frac{\partial u}{\partial \bm {n}}(\bm {x}) = g_N(\bm {x}), \bm {x}\in \partial\Omega_N$
with $\partial\Omega = \partial \Omega_D \cup \partial\Omega_N$ and $\partial\Omega_D \cap \partial \Omega_N = \emptyset$, with $\bm {n}$ representing the outward unit normal.

\begin{wrapfigure}{r}{0.46\columnwidth}
  \vspace{-10pt}
  \centering
  \includegraphics[width=0.46\columnwidth]{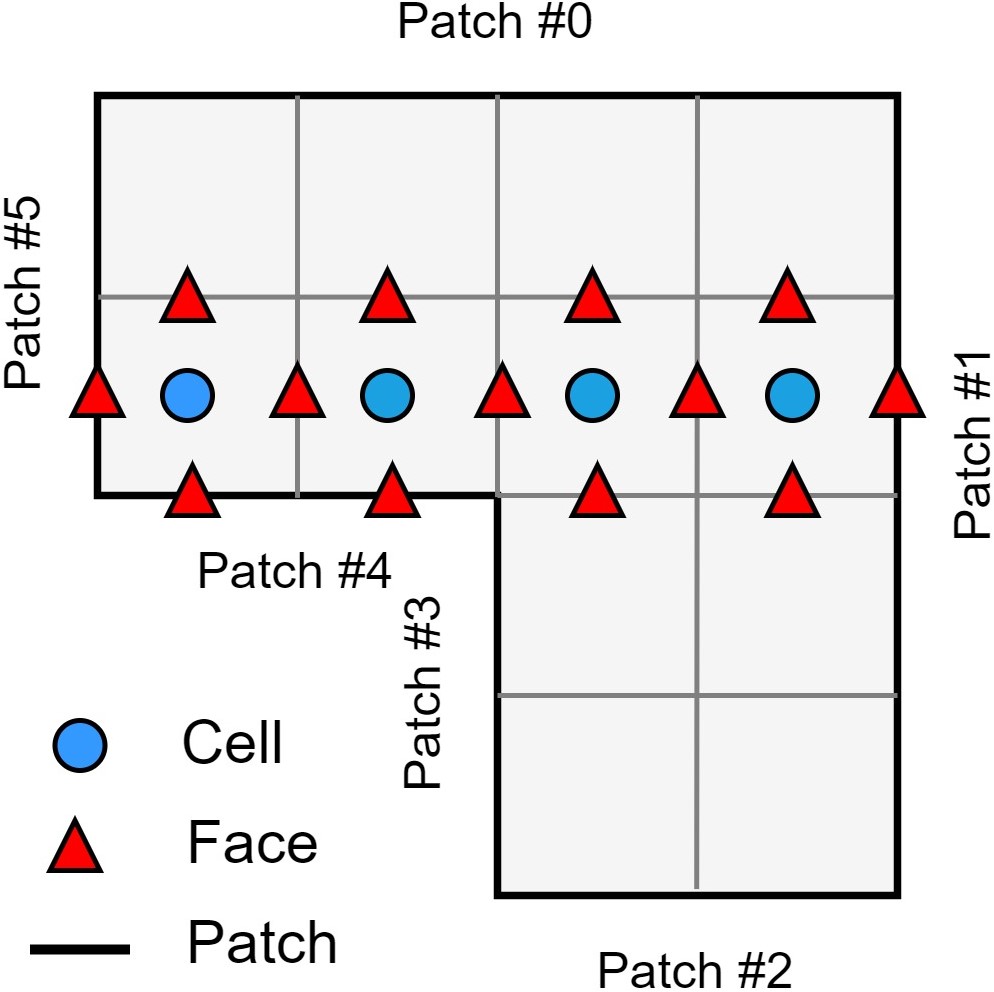}
  \caption{Cell-centered FVM mesh: cell, face, and patch.}
  \label{fig:fvm_cell_face}
  \vspace{-8pt}
\end{wrapfigure}
To obtain a discrete numerical solution, we discretize the domain and the governing equation using a cell-centered finite volume method (FVM). As illustrated in Fig.~\ref{fig:fvm_cell_face}, the domain $\Omega$ is partitioned into a set of control volumes $\{V_i\}_{i=1}^{N}$, where the primary unknowns are stored at the corresponding cell centroids $\{\bm{x}_i\}_{i=1}^{N} \in \Omega$. BCs are imposed at the boundary-face centroids $\{\bm{y}_i\}_{i=1}^{N_b} \in \partial \Omega$. Applying the finite-volume integration and standard flux approximations transforms the continuous PDE into a sparse linear system:
\begin{equation}
\bm A \bm u = \bm b
\end{equation}
where $\bm u \in \mathbb{R}^{N}$ is the vector of discrete cell-centered solutions, $\bm A \in \mathbb{R}^{N\times N}$ is the discrete Laplacian operator (incorporating mesh geometry and boundary treatment), and $\bm b \in \mathbb{R^{N}}$ aggregates the contributions from the source term $f$ and the boundary data $g$. The resulting sparse system is typically solved using an iterative Preconditioned Conjugate Gradient (PCG) method, which often requires hundreds or thousands of iterations to reach full convergence.

In this paper, we aim to learn a neural solution operator $\mathcal{N}_\theta$ that maps the domain geometry $\Omega$, boundary condition $B$, and the forcing field $f$ directly to the corresponding solution field, i.e.,
\begin{equation}
\mathcal{N}_\theta: (\Omega, B, f)\mapsto \bm u.
\end{equation}

\subsection{The Overview of \textsc{NPSolver}}

We propose \textsc{NPSolver}, a label-free neural Poisson solver trained via iterative physics supervision (Fig.~\ref{fig:model}). For each training instance, we sample a triplet $(\Omega,B,f)$ on the fly, where $\Omega$ is a geometry, $B$ is a boundary condition, and $f$ is a forcing field. Given $(\Omega,B,f)$, the network $\mathcal{N}_\theta$ predicts a cell-centered solution $\hat{\bm u}$. Unlike directly using PDE residual~\cite{pino}, \textsc{NPSolver} conducts $K$ steps of PCG initialized at $\hat{\bm u}$ to obtain a self-supervision target $\tilde{\bm u}=F_K(\hat{\bm u})$, and minimizes $\|\hat{\bm u}-\tilde{\bm u}\|_2^2$ while stopping gradients through $F_K$.
This objective enjoys stronger theoretical guarantees and alleviates the ill-conditioning issues that arise when directly minimizing the PDE residual (see Sec.~\ref{sec:loss}). To improve generalization under irregular meshes and mixed BCs, we introduce a Boundary-Aware Transolver (\textsc{BA-Transolver}).
Unlike \textsc{Transolver}~\cite{wu2024Transolver}, \textsc{BA-Transolver} tokenizes interior and boundary nodes separately and performs attention jointly,  which improves modeling under mixed BCs and geometry shifts through boundary--interior interactions (Sec.~\ref{sec:architecture}).

\subsection{Model Architecture}
\label{sec:architecture}
\begin{figure*}[t]
    \centering
    \includegraphics[width=0.99\textwidth]{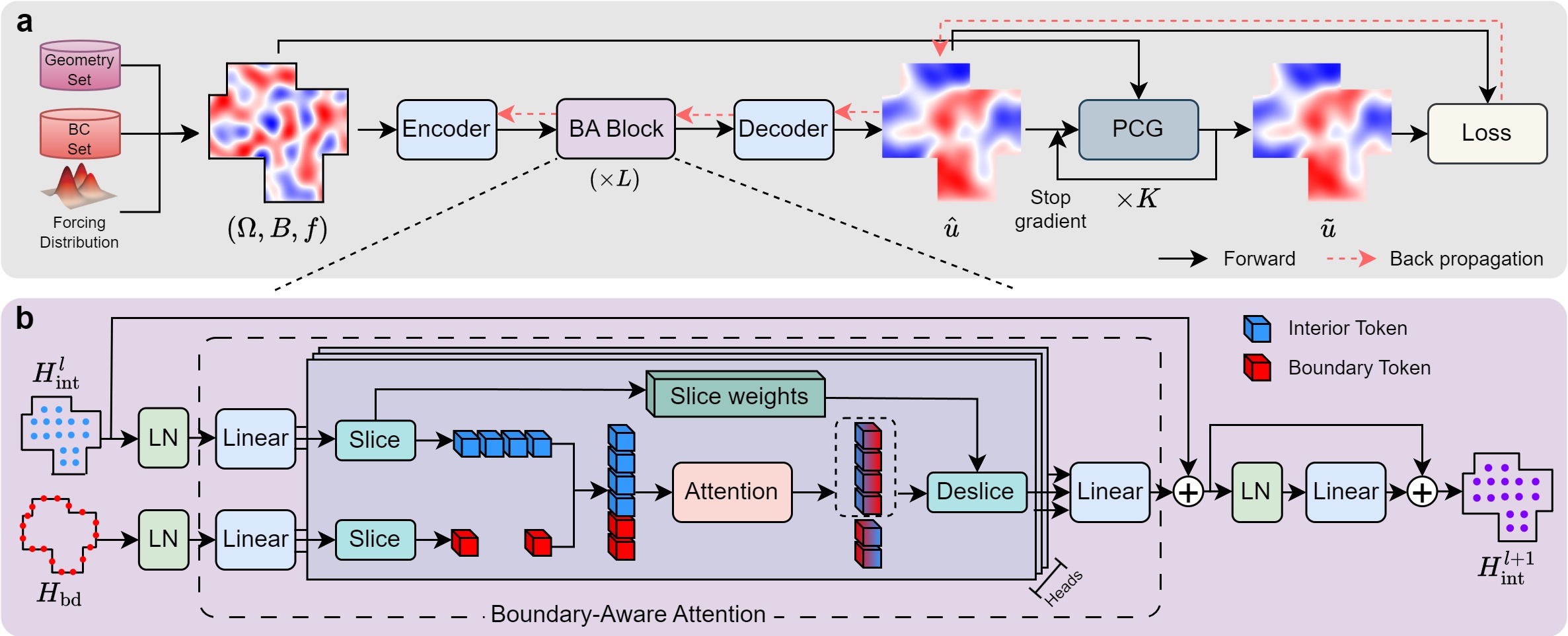}
    \caption{Overview of \textsc{NPSolver}. (a) Iterative Physics Supervision: The network predicts $\hat{\bm u}$ for a sampled instance ($\Omega, B, f$), which is refined by $K$ PCG steps to generate a self-supervision target $\tilde{\bm u}$ (with stop-gradient). (b) \textsc{BA-Transolver} Architecture: Interior and boundary nodes are respectively tokenized into $\bm Z_{int},\ \bm Z_{bd}$ and attended jointly to model boundary-interior interactions.}
    \label{fig:model}
    \vspace{-5pt}
\end{figure*}

Our network architecture (\textsc{BA-Transolver}) follows an encoder-block-decoder design (Fig.~\ref{fig:model}). The encoder maps physical inputs to latent node embeddings, a stack of $L$ Boundary-Aware (BA) blocks performs global message passing in a compact token space, and the decoder projects the updated embeddings back to the solution. 

The core innovation is the Boundary-Aware Attention module within each BA block. This module extends the "physics-attention" mechanism of \textsc{Transolver}~\cite{wu2024Transolver} by introducing two separate token streams. Unlike the original \textsc{Transolver}, which applies a shared slicing operation across all nodes, our approach separately tokenizes interior and boundary nodes before performing attention over their concatenation. This explicitly exposes boundary information in the token space, significantly improving sensitivity to mixed BCs and geometric shifts.

Given a cell-centered finite-volume mesh of domain $\Omega$, we define interior nodes $\{\bm{x}_i\}_{i=1}^{N}$ at cell centroids and boundary nodes $\{\bm{y}_j\}_{j=1}^{N_b}$ at boundary-face centroids. We construct interior node features $\bm p_i$ from spatial coordinates, and the forcing $f$, and boundary node features $\bm{q}_j$ from spatial coordinates, BC type, and the prescribed values. The initial stage of the \textsc{BA-Transolver} involves mapping physical features into a latent space while maintaining the distinction between the domain interior and its boundaries. Two encoders produce interior embeddings $\bm h^{\text{int}}_i=\text{Enc}_{\text{int}}(\bm p_i)$ and boundary embeddings $\bm h^{\text{bd}}_j=\text{Enc}_{\text{bd}}(\bm q_j)$, respectively. Collecting these embeddings into matrices $\bm H_\text{int} \in \mathbb{R}^{N\times d}$ and $\bm H_\text{bd} \in \mathbb{R}^{N_b\times d}$, we learn two sets of slice-assignment weights for interior and boundary nodes, $\bm W_{\text{int}} \in \mathbb{R}^{N\times S_\text{int}}$ and $\bm W_{\text{bd}} \in \mathbb{R}^{N_b\times S_\text{bd}}$, via two separate projection heads:
\begin{equation}
    \bm{W}_\text{int} = \text{Softmax}(\bm \Phi_\text{int}(\bm H_{\text{int}})), \quad \bm{W}_\text{bd} = \text{Softmax}(\bm \Phi_\text{bd}(\bm H_{\text{bd}})),
\end{equation}
where the softmax is applied along the slice dimension. This dual-stream approach ensures that boundary information is not overlooked by the much larger number of interior nodes. Using these assignments, we aggregate node embeddings into slice tokens by weighted pooling:
\begin{equation}
    \bm Z_\text{int} = \text{Pooling}(\bm W_\text{int}, \bm H_\text{int}), \quad
    \bm Z_\text{bd} = \text{Pooling}(\bm W_\text{bd}, \bm H_\text{bd})
\end{equation}
where $\text{Pooling}(\bm W, \bm H)$ denotes normalized weighted sums over nodes. We concatenate the two token sets and apply multi-head self-attention (MHA) to achieve boundary--interior interactions:
\begin{equation}
    \bm Z = [\bm Z_\text{int}; \bm Z_\text{bd}],\quad \bm Z' = \text{MHA}(\bm Z).
\end{equation}
Finally, we update only the interior node embeddings by deslicing from the updated interior tokens $\bm Z'_\text{int}$:
\begin{equation}
    \bm H'_\text{int} = \text{Deslice}(\bm W_\text{int}, \bm Z'_\text{int}).
\end{equation}
In this design, boundary tokens act as a persistent conditioning context. They inject boundary information into the global interactions while keeping the boundary representation anchored to the prescribed BCs. Each BA block follows a pre-norm transformer structure with residual connections and a feed-forward layer.

\subsection{Iterative Physics Supervision}
\label{sec:loss}

As described in Sec.~\ref{sec:problem}, our numerical reference pipeline discretizes the Poisson problem via a cell-centered FVM, yielding a sparse linear system $\bm A \bm u = \bm b$. While solutions are obtainable using an iterative PCG solver, challenging geometries often require a large number of iterations to satisfy strict convergence tolerances. We exploit this discretize-solve workflow to construct a physics supervision signal that requires only a small number of iterative steps during training.

Given an instance $(\Omega, B, f)$, the network predicts a cell-centered field $\hat{\bm u}=\mathcal{N}_\theta(\Omega,B,f)$. Rather than supervising $\hat{\bm u}$ with a computationally expensive converged label, we run $K$ steps of PCG initialized at $\hat{\bm u}$, and denote the resulting $K$-step update operator by $F_K(\cdot)$. We treat the $K$-step iterate as a pseudo-label,
\begin{equation}
\tilde{\bm u}=\operatorname{sg}\!\left(F_K(\hat{\bm u})\right),
\end{equation}
where $\operatorname{sg}(\cdot)$ denotes the stop-gradient operator. The model is trained by minimizing the discrepancy between the prediction and its PCG-updated counterpart:
\begin{equation}
\label{eq:loss_func}
L_{\text{iter}}(\theta)=\left\|\hat{\bm u}-\tilde{\bm u}\right\|_2^2.
\end{equation}
During training, PCG acts purely as a target generator, where gradients are not propagated through $F_K(\cdot)$, and backpropagation is performed only through the neural network.

Minimizing $L_{\text{iter}}$ enforces a self-consistency condition induced by the iterative solver, i.e., $\hat{\bm u}\approx F_K(\hat{\bm u})$. Near convergence, additional PCG steps produce only small corrections; conversely, when $\hat{\bm u}$ is far from satisfying $\bm A\bm u=\bm b$, the PCG update provides a nontrivial, physics-consistent correction direction. In this way, iterative physics supervision trains the network to output a state that is already close to the converged solution in a single forward pass, amortizing a substantial portion of the iterative solve into inference.

One common physics-informed baseline directly minimizes the discrete residual, e.g., $L_{\text{res}}=\|\bm A\hat{\bm u}-\bm b\|$. In our setting, $\bm A$ can be ill-conditioned on irregular meshes, making residual-based optimization slow and unstable due to poorly scaled gradients. In contrast, $L_{\text{iter}}$ learns from the solver's update in the solution space, which implicitly incorporates the effect of preconditioning and yields a more effective training signal toward the converged manifold. We analyze these properties theoretically in Sec.~\ref{sec:theory} and validate them empirically in Sec.~\ref{sec:ablation}.

{Although we instantiate iterative supervision with PCG in the Poisson/SPD setting, the high-level idea only requires a truncated solver map $F_K$. In principle, $F_K$ could be replaced by other Krylov solvers matched to the algebraic structure of the system, such as GMRES or MINRES, although our current theory is specific to the SPD case. We use a Jacobi preconditioner in this work because it is simple, cheap, and robust on varying irregular meshes, which helps isolate the effect of iterative supervision itself. Stronger SPD preconditioners may further improve conditioning and reduce the required K, but they also introduce additional setup cost and solver-specific complexity.}


\section{Theory}
\label{sec:theory}
This section establishes three theoretical results: (i) fixed-point consistency of the PCG-$K$ map, (ii) the analysis of stop-gradient, and (iii) the self-consistency residual is a well-conditioned error proxy. All proofs are provided in Appendix~\ref{app:theory}.

\subsection{Setup}
After discretization, each Poisson equation yields a linear system $\bm{A}\bm{u}=\bm{b}$. Throughout this section, we assume:

\textit{(A1)} $\bm{A}$ is symmetric positive definite (SPD);

\textit{(A2)} the preconditioner $\bm{M}$ used by PCG is SPD.

Define the self-consistency residual $\bm{s}_K(\bm{u}) := \bm{u}-F_K(\bm{u})$. Let $\bm{u}^\star := \bm{A}^{-1}\bm{b}$ be the unique solution. For SPD $\bm{A}$, define the energy norm $\|\bm{x}\|_{\bm{A}} := \sqrt{\bm{x}^\top\bm{A}\bm{x}}$.
Define the preconditioned matrix $\bm{C} := \bm{M}^{-1/2}\bm{A}\bm{M}^{-1/2}$
and $\kappa := \kappa(\bm{C}) = \lambda_{\max}(\bm{C})/\lambda_{\min}(\bm{C})$.
Accordingly, let
\begin{equation}
\rho := \frac{\kappa-1}{\kappa+1}\in [0,1).
\end{equation}

\subsection{Fixed-point Consistency}
A basic concern for solver-induced self-supervision is whether the induced objective introduces spurious minimizers.
The following result shows that the only fixed points of the finite-iteration PCG map are exact solutions of the linear system.

\begin{theorem} Assume A1--A2 and $K\ge 1$. Then for any $\bm{u}\in\mathbb{R}^N$,
\begin{equation}
F_K(\bm{u})=\bm{u}\quad\Longleftrightarrow\quad \bm{A}\bm{u}=\bm{b}.
\end{equation}
\end{theorem}

\subsection{Analysis of the Stop-gradient Update}
Our objective uses PCG as a target generator and blocks gradients through $F_K$.
To clarify why this matters, consider optimizing $\bm{u}$ directly using the iterative physics residual $\bm{s}_K(\bm{u})=\bm{u}-F_K(\bm{u})$.
The stop-gradient training corresponds to the semi-gradient update
$\bm{u}_{t+1}=\bm{u}_t-\eta\,\bm{s}_K(\bm{u}_t)$, which equals a relaxed fixed-point iteration. 

\begin{theorem}
Assume A1--A2 and $K\ge 1$. Consider the iteration
\[
\bm u_{t+1}=\bm u_t-\eta\,\bm s_K(\bm u_t)
=(1-\eta)\bm u_t+\eta F_K(\bm u_t),
\qquad \eta\in(0,1].
\]
Let $\bm e_t=\bm u_t-\bm u^\star$. Then
\[
\|\bm e_{t+1}\|_{\bm A}\le \bigl(1-\eta(1-\rho^K)\bigr)\,\|\bm e_t\|_{\bm A}.
\]
\end{theorem}

\begin{remark}
The theorem establishes a global contraction guarantee for the stop-gradient dynamics. Furthermore, if a solver map $F_K$ is differentiable at $\bm{u}^\star$ with Jacobian
$\bm{J}=\nabla F_K(\bm{u}^\star)$, then for the stop-gradient update we have
\[
\bm{e}_{t+1}=\bigl((1-\eta)\bm{I}+\eta\bm{J}\bigr)\bm{e}_t + O(\|\bm{e}_t\|_2^2),
\]
and for the full-gradient update on $\phi(\bm{u})=\tfrac12\|\bm{s}_K(\bm{u})\|_2^2$,
\[
\bm{e}_{t+1}=\Bigl(\bm{I}-\eta(\bm{I}-\bm{J})^\top(\bm{I}-\bm{J})\Bigr)\bm{e}_t
+ O(\|\bm{e}_t\|_2^2).
\]
Therefore, the full-gradient direction multiplies the residual by an extra factor
$(\bm I-\bm J)^\top$, yielding an effective curvature $(\bm I-\bm J)^\top(\bm I-\bm J)$.
When $F_K$ is close to identity near $\bm u^\star$ (so that $\bm J\approx \bm I$),
this can make the full-gradient step much weaker and more sensitive to the step size.
\end{remark}

\subsection{A Well-conditioned Error Proxy}
The following theorem provides a uniform equivalence between $\|\bm{u}-F_K(\bm{u})\|_{\bm A}$ and the energy error $\|\bm u-\bm u^\star\|_{\bm A}$, and contrasts it with the scaling behavior of the PDE residual $\bm r(\bm u)=\bm A\bm u-\bm b$.

\begin{theorem}
Assume A1--A2 and $K\ge 1$. Then for any $\bm{u}$,
\begin{equation}
    (1-\rho^K)\,\|\bm{u}-\bm{u}^\star\|_{\bm{A}}
\le \|\bm{s}_K(\bm{u})\|_{\bm{A}}
\le (1+\rho^K)\,\|\bm{u}-\bm{u}^\star\|_{\bm{A}}.
\end{equation}
Furthermore, let $\bm r(\bm u):=\bm A\bm u-\bm b$. Then
\begin{equation}
\lambda_{\min}(\bm A)\,\|\bm{u}-\bm{u}^\star\|_{\bm A}
\le \|\bm r(\bm u)\|_{\bm A}
\le \lambda_{\max}(\bm A)\,\|\bm{u}-\bm{u}^\star\|_{\bm A}.    
\end{equation}
\end{theorem}

\begin{remark} The theorem implies that $\|\bm s_K(\bm u)\|_{\bm A}$ tracks the true energy error $\|\bm u-\bm u^\star\|_{\bm A}$ up to a factor controlled by $\rho^K$; hence it provides a well-scaled solution-space correction (especially when $\rho^K$ is small). By contrast, the residual $\bm A\bm u-\bm b$ may be heavily rescaled by the spectrum of $\bm A$, and the squared-residual objective further suffers squared-conditioning since its curvature is $\bm A^2$.
\end{remark}

\section{Experiments}
In this section, we present a comprehensive experimental evaluation of \textsc{NPSolver} to assess its performance under challenging domain geometries and BCs, including
\begin{itemize}[leftmargin=12pt]
    \item \textit{2D generalization}: Tests on irregular domains to evaluate generalization across domain geometry, BCs, and forcing fields.
    \item \textit{3D scalability}: Performance assessment on 3D domains with varying internal topologies and mixed BCs.
    \item \textit{Downstream application}: Demonstration of \textsc{NPSolver} as a differentiable surrogate for gradient-based thermal control.
    \item \textit{Efficiency and ablation}: Quantitative analysis of computational costs and various ablation studies.
\end{itemize}

\sloppy \textsc{NPSolver} is trained with the iterative physics supervision and optimized with Adam and a OneCycle learning-rate schedule. We compare \textsc{NPSolver} against representative baselines spanning physics-informed and data-driven neural surrogates, including \textsc{PI-DeepONet}~\cite{wang2021pideeponet}, \textsc{PINO}~\cite{pino}, \textsc{PINN}~\cite{raissi_pinn}, \textsc{Transolver}~\cite{wu2024Transolver}, \textsc{Transolver++}~\cite{luo2025transolverplus}, \textsc{MGN}~\cite{pfaff_mgn}, \textsc{GPS}~\cite{rampasek2022GPS}, \textsc{PointNet++}~\cite{pointnet2}, and {\textsc{BENO}~\cite{wang2024beno}}. Ground-truth solutions are generated using an FVM coupled with a PCG solver. We report the relative $L_2$ error and inference time to assess model accuracy and computational cost.

\subsection{2D Generalization under Different BCs}
\label{sec:2d_exp}
In this study, we evaluate \textsc{NPSolver} on 2D Poisson problems to assess generalization under domain geometries, BCs, and forcing fields. All models are trained on a fixed geometry distribution and tested on different domain geometries.

\textbf{Geometry.} Following BENO~\cite{wang2024beno}, we construct a corner-removed square family with five categories, denoted as {C$k$}, where $k \in \{0, \dots, 4\}$ is the number of removed rectangular corners (C0 corresponds to the intact square). For each domain geometry, the width and height of the removed corner rectangles are sampled independently from a uniform distribution (detailed in Appendix~\ref{app_sec:2d}). To evaluate both in-distribution and out-of-distribution
generalization ability, we \emph{train} the model only on the most irregular category C4 and \emph{test} it on all categories C0-C4.

\textbf{BC regimes.} We consider three boundary regimes, which define three cases that share the same geometry family and differ only in boundary specification: (i) \textbf{All-Dirichlet}, Dirichlet BCs with zero value on all boundary patches; (ii) \textbf{All-Neumann}, zero-flux Neumann BCs on all boundary patches; and (iii) \textbf{RandomBC}, patch-wise mixed BCs where each boundary patch is independently and randomly assigned Dirichlet or Neumann type (each patch corresponds to a geometric segment, as illustrated in Fig.~\ref{fig:fvm_cell_face}). Empirically, these cases exhibit increasing learning difficulty.

In each case, training samples $(\Omega, B, f)$ are generated on-the-fly by sampling a domain $\Omega$ from the above geometry set, assigning BCs $B$ per the case definition, and drawing a forcing field $f$ from the distribution (Appendix~\ref{app_sec:2d}). For evaluation, we generate 100 test samples per category, totaling 500 samples per case.

\begin{table}[!t]
    \centering
    \caption{Relative $L_2$ error (\%) and inference time (ms) across the various geometries (C4 to C0) in All-Dirichlet BCs.}
    \label{tab:dirichlet_err}
    \begin{tabular}{ccccccc}
         \toprule
         Method & C4 & C3 & C2 & C1 & C0 & Time \\
         \midrule
         \textsc{PI-DeepONet} & >100 & >100 & >100 & >100 & >100 & 56.5 \\
         \textsc{PINO} & 19.40 & 23.67 & 23.88 & 25.46 & 26.53 & \textbf{5.5} \\
         \textsc{PINN} & \underline{5.74} & \underline{3.07} & \underline{3.28} & \textbf{1.13} & \textbf{0.77} & 251880 \\
         \midrule
         \textsc{NPSolver} & \textbf{1.58} & \textbf{2.20} & \textbf{2.88} & \underline{3.34} & \underline{3.59} & \underline{9.1} \\
         \bottomrule
    \end{tabular}
        \vspace{-6pt}
\end{table}

\begin{table}[!t]
    \centering
    \caption{Relative $L_2$ error (\%) and inference time (ms) across the various geometries (C4 to C0) in All-Neumann BCs.}
    \label{tab:neumann_err}
    \begin{tabular}{ccccccc}
         \toprule
         Method & C4 & C3 & C2 & C1 & C0 & Time \\
         \midrule
         \textsc{PI-DeepONet} & >100 & >100 & >100 & >100 & >100 & 55.7 \\
         \textsc{PINO} & \underline{34.66} & 32.35 & 36.43 & 34.64 & 37.84 & \textbf{5.9} \\
         \textsc{PINN} &  50.52 & \underline{27.25} & \underline{11.74} & \underline{24.79} & \textbf{0.19} & 281860 \\
         \midrule
         \textsc{NPSolver} & \textbf{3.11} & \textbf{4.62} & \textbf{7.95} & \textbf{11.27} & \underline{13.32} & \underline{9.0} \\
         \bottomrule
    \end{tabular}
        \vspace{-6pt}
\end{table}

\begin{figure}[!t]
  \centering
  \includegraphics[width=\linewidth]{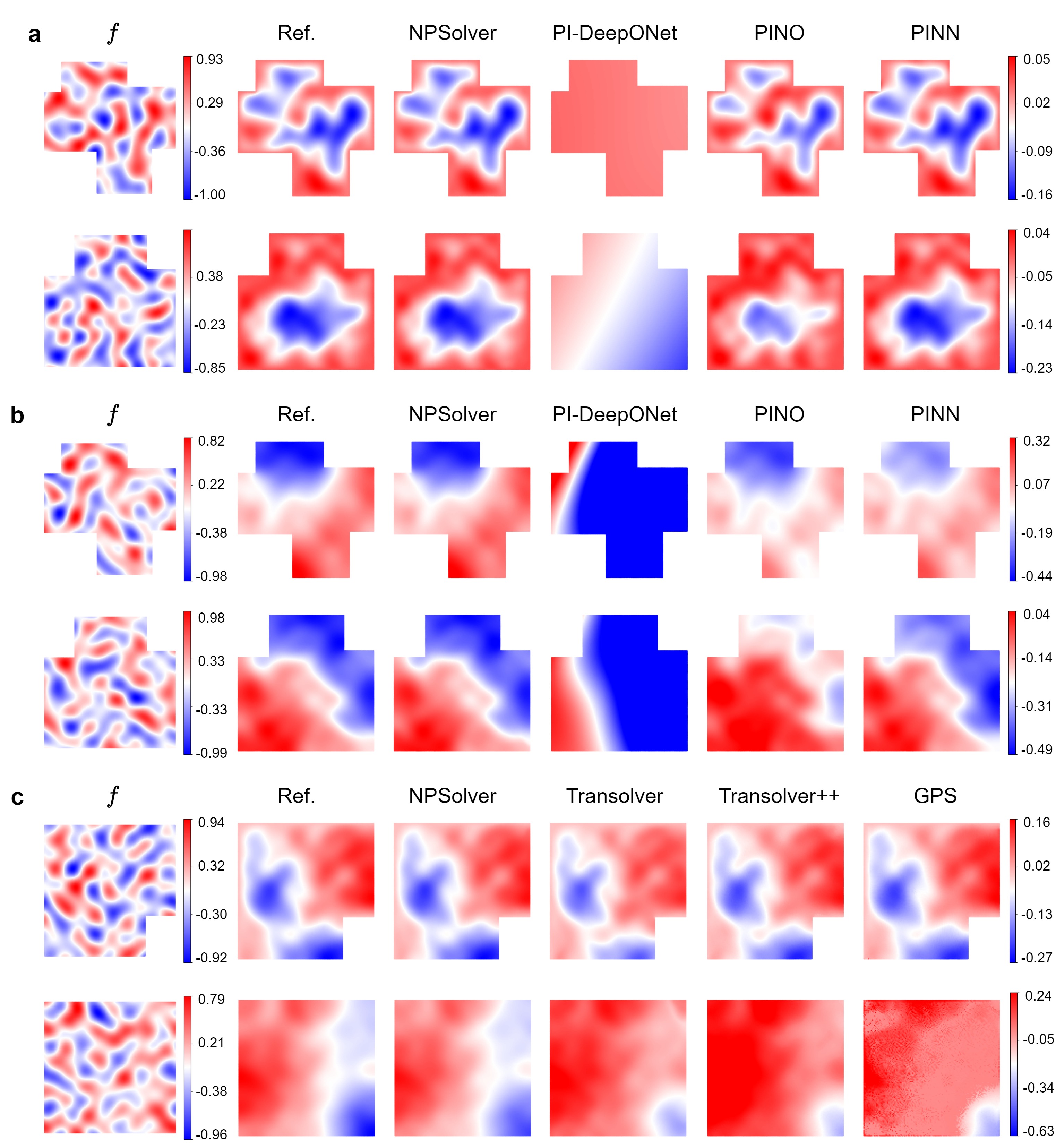}
  \caption{Visualization of models' predictions for representative samples. (a) All-Dirichlet case (C4 and C2). (b) All-Neumann case (C4 and C2). (c) RandomBC case (C1 and C0).}
  \label{fig:2d_vis}
      \vspace{-6pt}
\end{figure}

\begin{figure*}[!t]
    \centering
    \includegraphics[width=\textwidth]{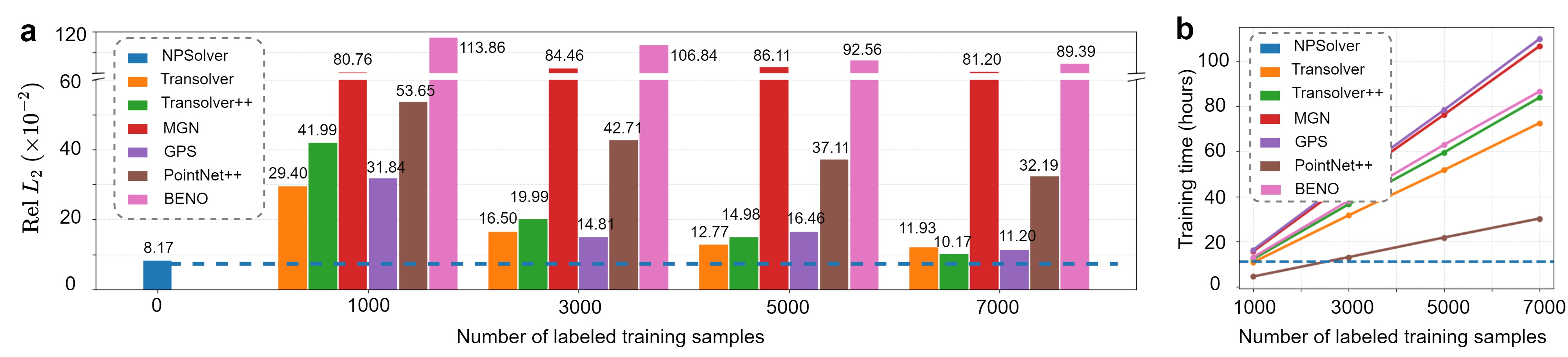}
    \caption{Data efficiency and scaling analysis in the RandomBC case. (a) Mean relative $L_2$ error as a function of the number of labeled training samples {for supervised data-driven baselines trained on the C4 geometry category and \textsc{NPSolver} is trained without solution labels}. (b) Comparison of training time {under the same labeled-data budgets.}}
    \label{fig:data_driven}
    \vspace{-6pt}
\end{figure*}

\begin{figure}[!t]
    \centering
    \includegraphics[width=\linewidth]{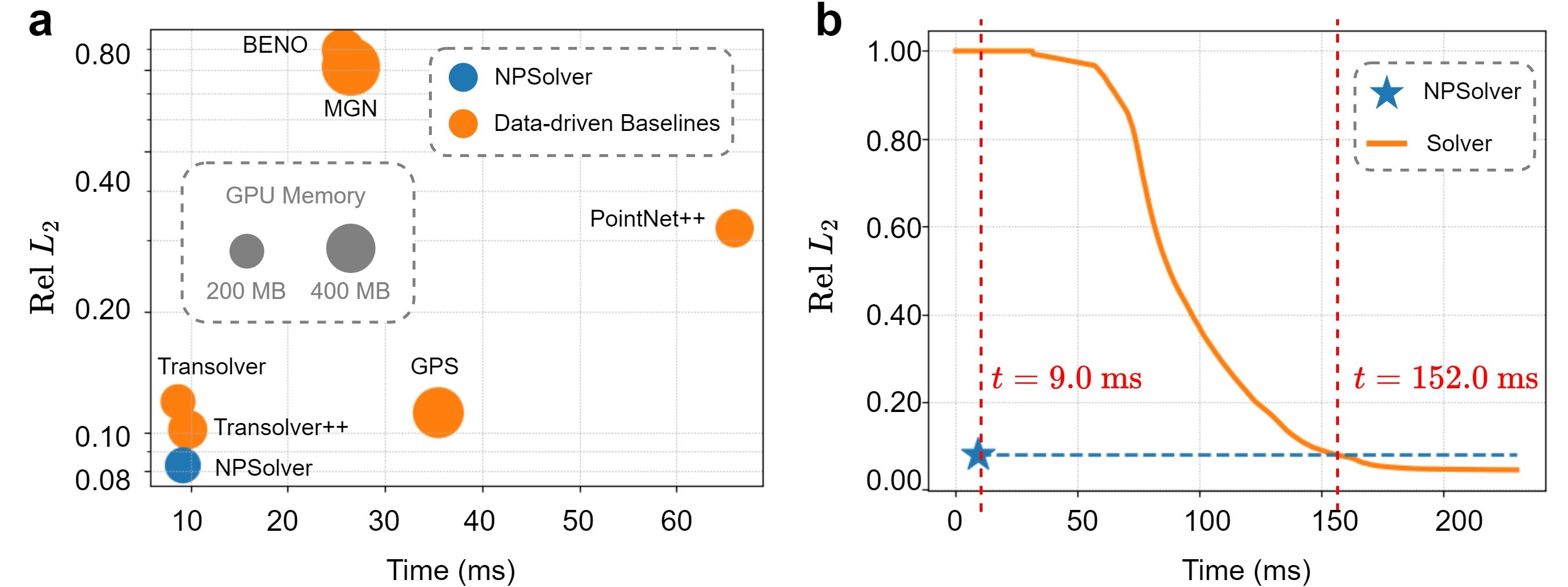}
    \caption{Computational Efficiency {in the RandomBC case}. (a) {Multi-objective comparison of predictive accuracy, inference time, and GPU memory usage.} (b) {Accuracy–time comparison between \textsc{NPSolver} and the iterative numerical solver.}}
    \vspace{-6pt}
    \label{fig:cost}
\end{figure}

\textbf{All-Dirichlet / Neumann results.} In the All-Dirichlet and All-Neumann cases, we compare \textsc{NPSolver} against three representative self-supervised baselines, including \textsc{PI-DeepONet}, \textsc{PINO} and \textsc{PINN}. Relative $L_2$ errors over the five geometry categories and average inference times are summarized in Table~\ref{tab:dirichlet_err} and Table~\ref{tab:neumann_err}. Visualizations of models' predictions are provided in Fig.~\ref{fig:2d_vis}.

Notably, \textsc{PI-DeepONet} implicitly assumes a fixed computational domain and struggles with joint generalization over both geometry and forcing. \textsc{PINO}, relying on an FNO backbone defined on regular grids, suffers from severe performance degradation under irregular geometry. \textsc{PINN} needs to retrain the network per sample, causing high optimization cost (Appendix~\ref{app_sec:exp} for further details).

Across both BC regimes, \textsc{NPSolver} consistently outperforms \textsc{PI-DeepONet} and \textsc{PINO} by a wide margin. In the All-Dirichlet case, our accuracy is comparable to that of the per-instance \textsc{PINN}, while providing orders-of-magnitude faster inference at test time. In the more challenging All-Neumann case, \textsc{NPSolver} retains strong performance across geometry shifts and outperforms \textsc{PINN} on most categories. Overall, these results demonstrate that \textsc{NPSolver} delivers a favorable accuracy–efficiency trade-off, achieving robust generalization across domain geometry under both Dirichlet and Neumann BCs, while maintaining fast inference.

\textbf{RandomBC results.} In the RandomBC case, the model needs to simultaneously generalize across domain geometries, forcing fields, and patch-wise varying BCs. Given the limitations of the three self-supervised baselines discussed above (e.g., fixed-domain assumptions, regular-grid constraints, and per-instance retraining), we instead compare \textsc{NPSolver} against representative data-driven baselines, including \textsc{Transolver}, \textsc{Transolver++}, \textsc{MGN}, \textsc{GPS}, \textsc{PointNet++}, and {\textsc{BENO}}. For these baselines, we construct supervised datasets of varying sizes (e.g., 1k/3k/5k/7k labeled samples) for the C4 geometry category with the numerical solver. In contrast, \textsc{NPSolver} is trained without solution labels via iterative physics supervision (Appendix~\ref{app_sec:exp} for further details).

Fig.~\ref{fig:data_driven} illustrates the test error and training time as functions of training dataset size, with qualitative visualizations in Fig.~\ref{fig:2d_vis}(c). {The reported training time for supervised baselines excludes the offline numerical label-generation cost.} Remarkably, \textsc{NPSolver} achieves a relative $L_2$ error of 8.17\% without any labeled data, outperforming the strongest data-driven baseline (10.17\%) even when the latter is trained with 7k labeled samples. Furthermore, \textsc{NPSolver} requires significantly less training time than data-driven baselines. This highlights the efficiency of iterative physics supervision: it provides a stronger inductive bias and markedly better sample efficiency than purely data-driven regression, while avoiding the cost of large-scale label generation (More visualizations in Appendix Fig.~\ref{app_fig:2d}).

\textbf{Computational cost.} In the RandomBC case, we assess the practical utility of \textsc{NPSolver} via a multi-objective analysis of predictive accuracy, inference time, and GPU memory usage. All measurements are conducted on identical GPU hardware to ensure a fair comparison. As shown in Fig.~\ref{fig:cost}(a), \textsc{NPSolver} occupies the optimal lower-left region, simultaneously achieving the lowest error and high-speed inference with minimal memory overhead. To quantify the acceleration, we compare \textsc{NPSolver} against the iterative numerical solver. As illustrated in Fig.~\ref{fig:cost}(b), \textsc{NPSolver} generates high-fidelity predictions in a single 9.0 ms forward pass, reaching an error level that the numerical solver attains after 152.0 ms. This represents an over 15$\times$ speedup at matched accuracy. These results highlight \textsc{NPSolver}'s robust generalization and its potential to serve as a high-efficiency alternative to classical iterative methods.
\vspace{-2pt}

\begin{figure}
  \centering
  \includegraphics[width=\linewidth]{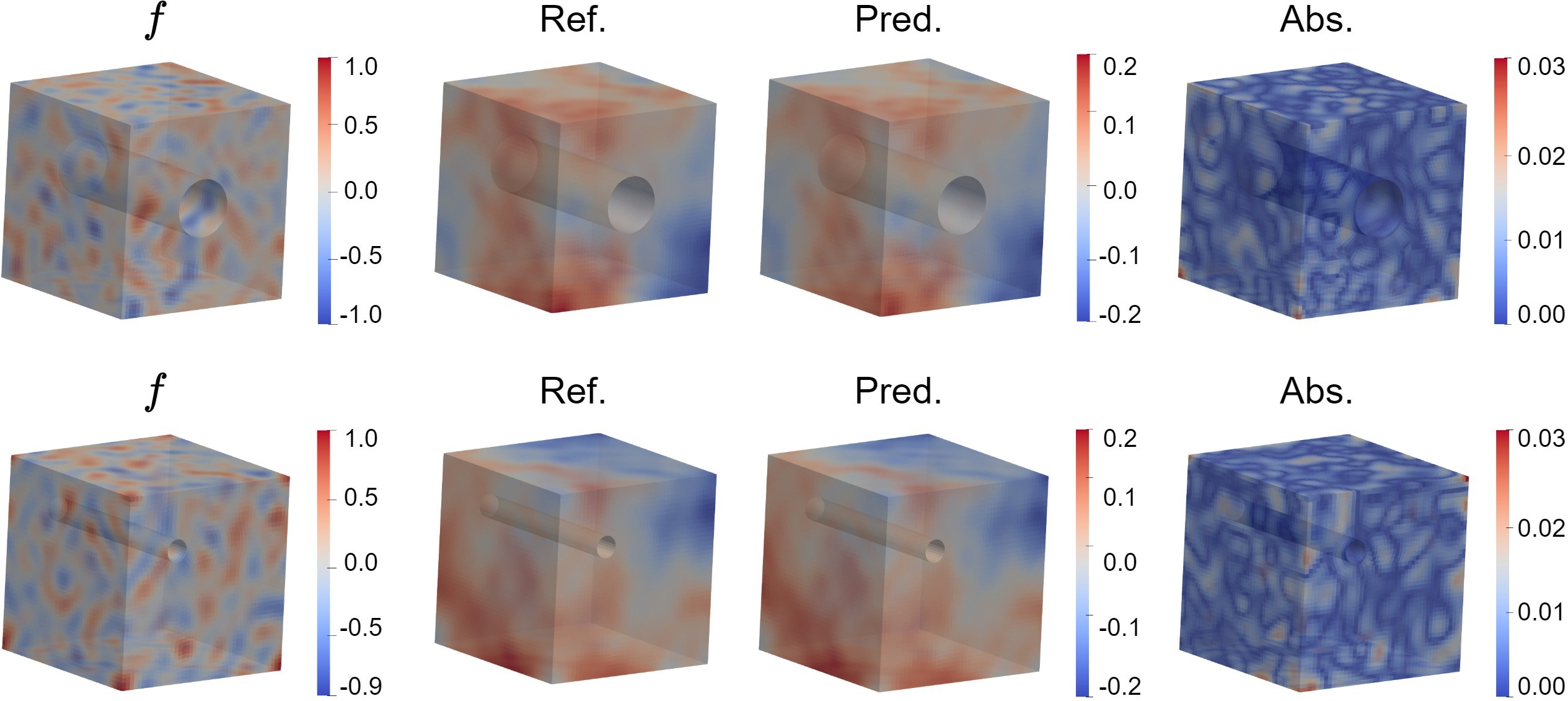}
  \caption{{Representative 3D predictions of \textsc{NPSolver} on the cube-with-cylindrical-hole Poisson problem with mixed BCs.}}
  \label{fig:3d_vis}
  \vspace{-6pt}
\end{figure}

\begin{figure}
  \centering
  \includegraphics[width=\linewidth]{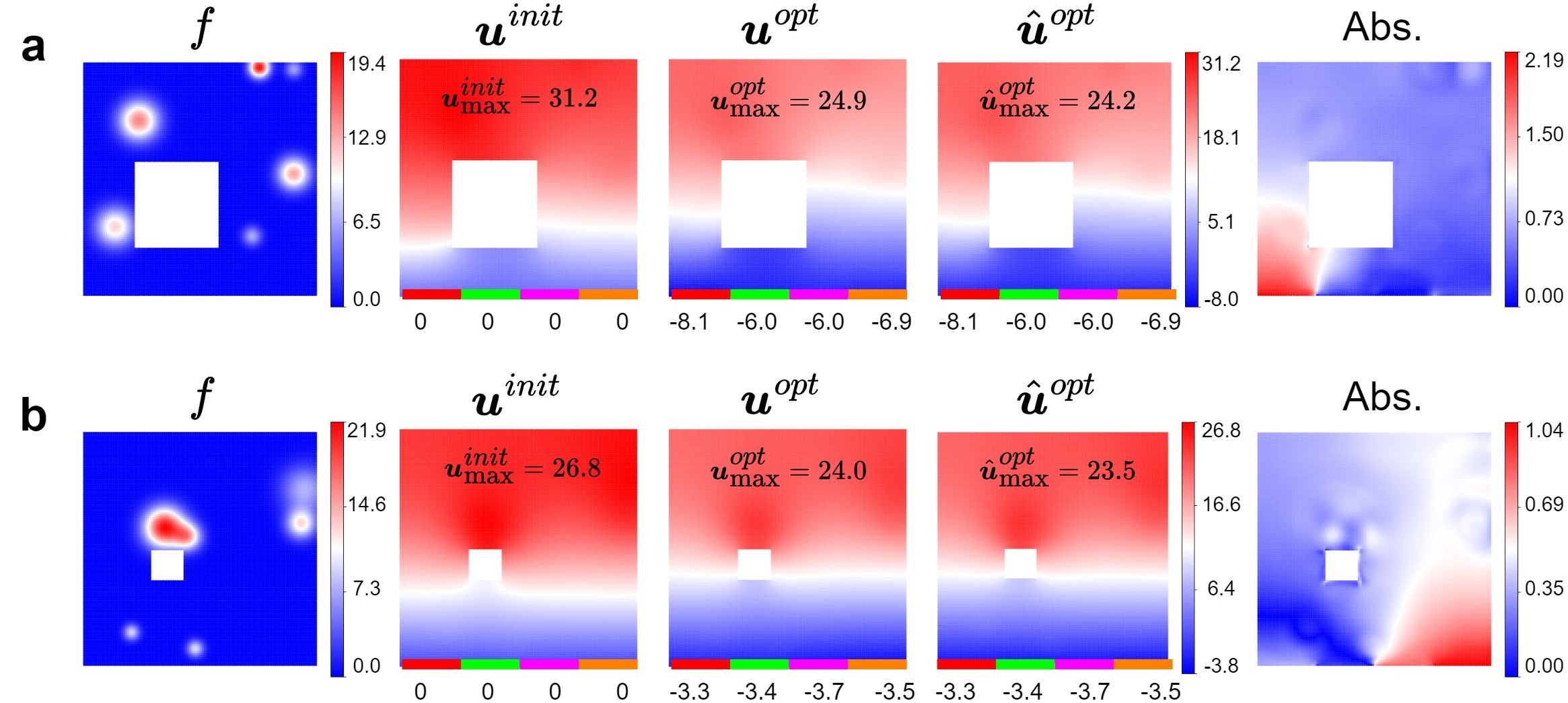}
  \caption{Visualization of the thermal control task of representative samples with forcing field $f$, the initial temperature $\bm u^{init} (\bm c= \bm 0)$,  the reference temperature $\bm u^{opt}$ and \textsc{NPSolver}'s prediction $\hat {\bm u}^{opt}$ under optimized boundary values $\bm c$.}
  \label{fig:control_inv_res}
  \vspace{-6pt}
\end{figure}

\subsection{3D Poisson on Cube-with-Cylindrical-Hole}
In this section, we evaluate the scalability of \textsc{NPSolver} on a 3D Poisson problem posed on irregular domains with mixed BCs. The task requires the model to generalize over both forcing fields and a parameterized family of geometries with internal cavities.

\textbf{Geometry and BCs.} We construct a family of 3D domain geometries by subtracting a cylindrical cavity from a cube. Both the diameter $d$ and the center position of the cylinder are randomly sampled (detailed in Appendix~\ref{app_sec:3d}). Regarding BCs, we impose zero-flux Neumann BCs on the outer cube surfaces, and zero-valued Dirichlet BCs on the internal cylindrical surface. This mixed-BC configuration couples an internal Dirichlet surface with an outer Neumann boundary, creating a challenging solution landscape.

\textbf{Results.} During training, \textsc{NPSolver} samples domain instances $\Omega$ from the geometry set described above and a forcing $f$ from the forcing distribution on the fly, and is trained with the iterative physics supervision. For evaluation, we generate 10 test domain instances, each paired with 10 independent forcing fields, yielding a total of 100 test samples with ground truth solutions provided by the numerical solver. On this 3D test set, \textsc{NPSolver} achieves a mean relative $L_2$ error of $8.67\%$ with an average inference time of 11.8 ms. {For comparison, a supervised \textsc{Transolver} baseline trained with 1k labeled samples achieves a relative $L_2$ error of 27.22\%}. Visualizations of model predictions are provided in Fig.~\ref{fig:3d_vis}. Overall, these results demonstrate that \textsc{NPSolver} successfully scales to 3D irregular domains and maintains robust accuracy under mixed BCs, {while substantially outperforming a representative supervised baseline}.

\subsection{Thermal Control on Perforated Plate}
In this section, we apply \textsc{NPSolver} to a thermal boundary-control problem on a perforated electronic plate. The objective is to maintain a steady-state temperature below a safety threshold while reducing the cooling effort. We consider steady-state heat conduction on a square plate $\Omega \subset \mathbb{R}^2$ with an internal rectangular cavity. Let $u(x)$ denote the temperature relative to the ambient environment, governed by the equation $\nabla^2u(\bm x) = - f(\bm x),\ \bm x \in \Omega$, where $f(\bm x)$ represents volumetric heat generation from electronic components. 

\textbf{Geometry and BCs.} We construct a family of 2D domains by subtracting a single rectangular cavity from a square domain. For each domain geometry, the cavity's width, height, and center coordinates are randomly sampled (detailed in Appendix~\ref{app_sec:control}). Controllable cooling is applied through four independently controlled Dirichlet segments on the bottom boundary, with values parameterized by a vector $\bm c = (c_1, c_2, c_3, c_4)$. All other boundaries are assigned zero-flux Neumann BCs, modeling insulated surfaces. 

\textbf{Results on forward problems.} To address this task, \textsc{NPSolver} is trained to learn the mapping $(\Omega, \bm c, f) \rightarrow u$. During training, $\Omega$ is sampled from the geometry set described above, $\bm c_i$ is sampled uniformly from $[-5, 5]$, and $f$ is sampled from a specially designed distribution (Appendix~\ref{app_sec:control}) to simulate realistic heating. For evaluation, we generate 10 test domain instances, each paired with 10 independent forcing and randomized Dirichlet values $\bm c$, yielding a total of 100 test samples with ground-truth solutions provided by the numerical solver. On the test set, the model achieves a mean relative $L_2$ error of 0.58\% and a peak temperature $\bm u_\text{max}$ relative error of 0.44\%. The inference time is only 10.6 ms, providing the fast and accurate prediction necessary for active thermal control.

\textbf{Results on control tasks.} We initialize the four Dirichlet segments with $\bm c=\bm 0$, i.e., no active cooling. If the predicted peak temperature $\hat {\bm u}_{\max}$ is below the prescribed threshold $u_m$, no control is required. Otherwise, we fix the network parameters and optimize $\bm c$ via gradient-based updates to enforce the peak-temperature constraint while penalizing excessive cooling. Specifically, the control loss is defined as
\begin{equation}
L_{\mathrm{ctrl}}(\bm{c})
=
\mathrm{ReLU}\!\left( \hat {\bm u}_{\max} - {u}_{m}\right)
\;+\;
\alpha \cdot \frac{1}{4}\sum_{i=1}^{4} c_i^{2},
\end{equation}
where $\alpha$ balances constraint satisfaction and cooling effort. 

We evaluate the control performance on the test set by optimizing the cooling vector $\bm c$ over 100 iterations with a safety threshold $u_m=25$. {And we do not explicitly constrain the control range during optimization.} The optimization achieved an 83\% success rate in satisfying the safety constraint. On average, the peak temperature is reduced from an initial $\bm u^{init}_{\max} = 33.42$ (with $\bm c = 0$) to $\bm u^{opt}_{\max} = 25.31$. This significant thermal reduction is achieved with a cooling cost $\frac{1}{4}\sum^4_{i=1}\vert c_i \vert$ of $4.47$ and an average optimization time of 5.91 s per instance, demonstrating that \textsc{NPSolver} serves as an efficient differentiable surrogate for gradient-based boundary control. {We further analyze the 17\% failed cases and find that all failures stem from out-of-distribution control requirements: satisfying the temperature constraint requires boundary values outside the surrogate’s training range [-5,5], rather than surrogate prediction error or optimization failure. This suggests that a broader training range of $\bm c$ may further improve robustness.}

As shown in Fig.~\ref{fig:control_inv_res}, we visualize the heat-source field $f$, the initial reference temperatures $\bm u^{init}$, and both the reference $\bm u^{opt}$ and \textsc{NPSolver}'s prediction $\hat {\bm u}^{opt}$ under optimized boundary values $\bm c$. Fig.~\ref{fig:control_curve} illustrates the control loss convergence for a representative test sample. \textsc{NPSolver} reaches a loss level comparable to the numerical solver in just 0.56 s, whereas the numerical solver requires 6.20 s, representing an over 10$\times$ speedup in the optimization process. These results suggest that \textsc{NPSolver} can effectively serve as a fast surrogate in such optimization tasks, offering a favorable trade-off between computational efficiency and predictive accuracy.

\begin{table}[t!]
    \centering
    \caption{Relative L2 error (\%) for combinations of two architectures and two training paradigms: iterative physics supervision (I.S.) and data supervision (D.S.).}
    \label{tab:ablation}
    \resizebox{8.35cm}{!}{
    \begin{tabular}{ccccccc}
    \toprule
        Method & C4 & C3 & C2& C1 & C0 & Avg.\\
        \midrule
        \textsc{BA-Transolver (D.S.)} & \textbf{3.03} & \textbf{6.00} & 11.20 & \underline{11.16} & 15.04 & \underline{9.29}\\
        \textsc{Transolver (D.S.)} & 5.49 & 8.17 & 12.08 & 15.36 & 18.56 & 11.93\\
        \textsc{Transolver (I.S.)} & 6.46 & 8.22 & \underline{9.44} & 11.27 & \underline{13.70} & 9.82\\
        \midrule
        \textsc{BA-Transolver (I.S.)} & \underline{4.91} & \underline{7.07} & \textbf{8.08} & \textbf{9.70} & \textbf{11.08} & \textbf{8.17}\\
        \bottomrule
    \end{tabular}}
    \vspace{-2pt}
\end{table}

\begin{table}[!t]
    \centering
    \caption{Relative $L_2$ error (\%) and training time (hours) for the residual supervision and iterative physics supervision with different PCG steps $K$.}
    \label{tab:ablation_k}
    \begin{tabular}{cccccc}
    \toprule
          & Residual & $K=1$ & $K=20$ & $K=40$ & $K=80$ \\
         \midrule
         Rel $L_2$ & 79.59 & 16.80 & \textbf{8.12} & \underline{8.17} & 8.90 \\
         Training time & 10.53 & 10.87 & 11.05 & 11.18 & 11.50 \\
         \bottomrule
    \end{tabular}
    \vspace{-2pt}
\end{table}
\subsection{Ablation Study}
\label{sec:ablation}
We conduct ablation studies on the 2D RandomBC case to evaluate the contributions of iterative physics supervision and \textsc{BA-Transolver}.

\textbf{Architecture and supervision paradigm.} We evaluate three new variants by crossing two architectures (\textsc{BA-Transolver} and \textsc{Transolver}) with two training paradigms: iterative physics supervision (I.S.) and data supervision (D.S.). Table~\ref{tab:ablation} reports per-category relative $L_2$ errors on RandomBC. \textsc{NPSolver}, i.e., \textsc{BA-Transolver} (I.S.), achieves the best performance. Under data supervision, \textsc{BA-Transolver} outperforms vanilla \textsc{Transolver}, indicating that explicitly separating boundary and interior during tokenization is beneficial for the complex boundary-value problems. Furthermore, under a fixed backbone, iterative physics supervision (I.S.) consistently outperforms data supervision (D.S.), suggesting that iterative supervision acts as a stronger inductive bias than pure regression in this boundary-sensitive regime, yielding better generalization.

\begin{table}[t]
    \centering
    \caption{{Per-optimizer-iteration cost breakdown for residual supervision and iterative physics supervision with different PCG steps $K$.}}
    \label{tab:cost_breakdown}
    \begin{tabular}{cccccc}
        \toprule
        & Residual & $K=1$ & $K=20$ & $K=40$ & $K=80$ \\
        \midrule
        Loss time (s) & 0.013 & 0.014 & 0.020 & 0.026 & 0.036  \\   
        Total time (s) & 0.376 & 0.377 & 0.383 & 0.389 & 0.399 \\
        Loss ratio (\%) & 3.46 & 3.71 & 5.22 & 6.68 & 9.02 \\
        \bottomrule
    \end{tabular}
\end{table}
\textbf{Impact of PCG steps in iterative supervision and residual supervision.} We investigate the influence of the number of PCG steps $K$ used to generate the iterative supervision target $\tilde{\bm u} = F_K(\hat{\bm u})$ and compare against direct residual supervision $\Vert \bm A \hat{\bm u} - \bm b \Vert$. As summarized in Table~\ref{tab:ablation_k}, residual supervision performs poorly in the complex RandomBC case. This is primarily due to the ill-conditioning of $\bm A$, which induces poorly scaled gradients that hinder convergence and destabilize the optimization process. In contrast, iterative supervision markedly improves accuracy by providing a
more stable and well-scaled training signal, while introducing only a small additional training cost. {A finer per-optimizer-iteration cost breakdown in Table~\ref{tab:cost_breakdown} shows that even at $K = 40$, the loss computation under iterative supervision accounts for only 6.68\% of each optimizer iteration, compared with 3.46\% under residual supervision, indicating that the dominant training cost still comes from the network forward/backward pass rather than the PCG steps.} Peak performance is achieved at an intermediate $K$. This trend reflects a practical trade-off: small $K$ yields a weak correction signal that may not sufficiently guide predictions toward a numerically consistent solution, whereas overly large $K$ can produce targets that are too aggressive relative to the current model output, increasing optimization difficulty and harming stability.

\textbf{Impact of stop-gradient.} We assess the necessity of stopping gradient through the $K$-step PCG operator $F_K(\cdot)$. Without stop-gradient, the relative $L_2$ error increases from 8.17\% to 11.90\%. This observation is consistent with Theorem~\ref{thm:stop_grad}, which suggests that backpropagating through the truncated PCG steps can lead to more sensitive updates, while applying stop-gradient yields a more stable learning signal.

\section{Conclusion}
We propose \textsc{NPSolver}, a neural Poisson solver trained \emph{without} solution labels via iterative physics supervision. Instead of relying on fully converged numerical solutions, \textsc{NPSolver} constructs self-supervision targets by applying a small, fixed number of PCG iterations to the network prediction. Our theory shows that the self-consistency residual induced by iterative supervision is a well-conditioned proxy for the solution error, and that the stop-gradient design leads to more favorable optimization dynamics. To improve generalization across domain geometries and BCs, we introduced \textsc{BA-Transolver}, a boundary-aware attention architecture that explicitly integrates boundary information when modeling global interactions. \textsc{NPSolver} is evaluated on a diverse suite of 2D and 3D experiments, demonstrating its ability to generalize across irregular geometries, BCs, and forcing fields, as well as its utility in a downstream thermal control application. Overall, these results show that \textsc{NPSolver} offers a favorable error–cost trade-off for complex boundary-value Poisson problems.

\section*{Limitations and Ethical Considerations}

In this paper, we focus on the Poisson equation as a representative PDE setting. More broadly, our method suggests an alternative approach to constructing physics-based training objectives beyond PDE residual losses. However, its generality is not fully established in this paper. Validating on broader PDE families and extending the current PCG/Jacobi instantiation to other Krylov solvers and stronger preconditioners remain important directions for future work. Our experiments are conducted on PDE datasets and do not involve human subjects or sensitive information.

\section*{GenAI Disclosure}
The authors utilized generative AI tools to polish the language and improve the text quality. All AI-generated suggestions were thoroughly reviewed and validated by the authors.


\begin{acks}
{The work is supported by the Beijing Natural Science Foundation (No. F261002) and the National Natural Science Foundation of China (No. 62276269 and No. 62506367). R.Z. would like to acknowledge the supported from the China Postdoctoral Science Foundation under Grant Number 2025M771582 and the Postdoctoral Fellowship Program of CPSF under Grant Number GZB20250408.}
\end{acks}

\bibliographystyle{ACM-Reference-Format}
\bibliography{reference}

\appendix
\setcounter{equation}{0}
\renewcommand{\theequation}{A.\arabic{equation}}
\renewcommand\theHequation{Appendix.\theequation}
\setcounter{figure}{0}
\renewcommand\theHfigure{Appendix.\thefigure}
\renewcommand{\thefigure}{A.\arabic{figure}}
\setcounter{table}{0}
\renewcommand{\thetable}{A.\arabic{table}}
\renewcommand\theHtable{Appendix.\thetable}

\begin{center}
\Large\bfseries Appendix
\end{center}

\section{FVM Discretization}
\label{app_sec:fvm}
We adopt a cell-centered finite-volume method (FVM) to discretize the Poisson equation on an irregular domain. As illustrated in Fig.~\ref{app_fig:fvm_mesh}, the computational domain is partitioned into a set of non-overlapping control volumes $\{ V_i \}^{N}_{i=1}$. The unknown $u$ and the forcing $f$ are stored at cell centroids $\{ \bm x_i \}$, while boundary conditions are imposed on boundary faces (grouped into boundary patches).

We consider the Poisson equation
\begin{equation}
    \nabla^2 u(\bm x) = f(\bm x), \bm x \in \Omega,
\end{equation}
with boundary conditions applied on $\partial \Omega$ (Dirichlet and/or Neumann).
Integrating the equation over each control volume $V_i$ and applying the divergence theorem yield
\begin{equation}
\label{app_eq:fvm_int}
    \int_{V_i} \nabla^2 u dV = \int_{V_i} f dV \Longrightarrow \int_{\partial V_i} \nabla u \cdot \bm n dS = \int_{V_i} f dV
\end{equation}
where $\partial V_i$ is the boundary of $V_i$ and $\bm n$ is the outward pointing unit vector normal to $\partial V_i$. Let $E(i)$ denote the set of faces of cell $i$, $\vert S_e \vert$ the area for each face $e \in E(i)$, and $\vert V_i \vert$ the volume for each control volume $V_i$. The integration in Eq.~\ref{app_eq:fvm_int} can be discretized as
\begin{equation}
\label{app_eq:fvm_flux_dis}
    \sum_{e \in E(i)} (\nabla u \cdot \bm n)_e \vert S_e \vert = f_i \vert V_i, \vert
\end{equation}
where $f_i$ denotes the cell-averaged forcing in $V_i$.

We approximate the normal gradient using neighboring cell values. On an orthogonal mesh, a standard approximation is
\begin{equation}
    (\nabla u \cdot \bm n)_e \approx \frac{u_{N(e)} - u_i}{d_e},
\end{equation}
where $u_i$ is the cell-center unknown in $V_i$, $u_{N(e)}$ is the neighbor-cell unknown across face $e$, and $d_e$ is the projected distance along the face normal. Therefore, Eq.~\ref{app_eq:fvm_flux_dis} can be rewritten as
\begin{equation}
    \sum_{e \in E(i)} \frac{u_{N(e)} - u_i}{d_e}  \vert S_e \vert = f_i \vert V_i \vert,
\end{equation}
Collecting all cells leads to a sparse linear system
\begin{equation}
    A \bm u = \bm b,
\end{equation}
where $A$ is the discrete Laplacian induced by the mesh and boundary treatment, and $b$ aggregates source and boundary contributions.

For Dirichlet boundary $u(\bm x) = g_D(\bm x)$ on a boundary face, the boundary value is incorporated into $\bm b$ (and/or modifies diagonal entries) through the face flux discretization. For Neumann boundary $\frac{\partial u}{\partial \bm n} = g_N(\bm x)$ on a boundary face, the prescribed flux contributes directly to the face integral. For the pure Neumann case, the system has a one-dimensional null space (solutions are defined up to an additive constant). We remove this ambiguity by imposing a reference constraint: selecting a reference location $\bm {x}_{\mathrm{ref}} \in \Omega$ and enforcing $u(\bm {x}_{\mathrm{ref}}) = 0$, which makes the linear system well-posed.

\begin{figure}
    \centering
    \includegraphics[width=\linewidth]{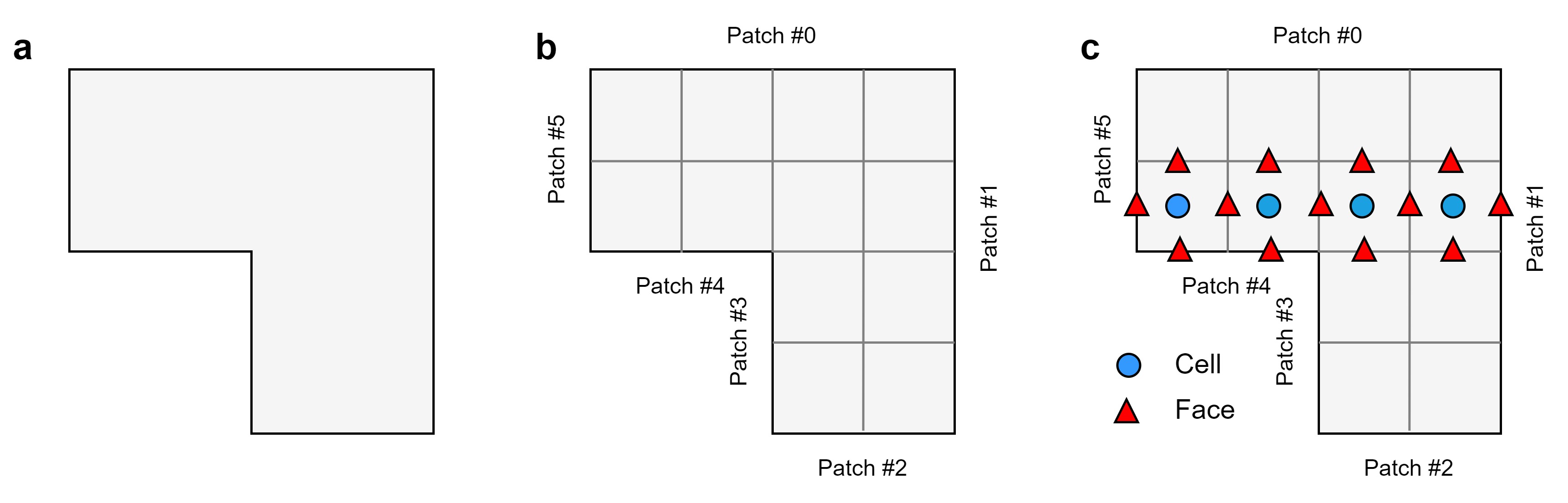}
    \caption{(a) Computational domain. (b) FVM computational mesh (\textit{quadrilateral}). (c) FVM cells and faces.}
    \label{app_fig:fvm_mesh}
\end{figure}

\section{Preconditioned Conjugate Gradient (PCG)}
\label{app_sec:pcg}
The FVM discretization yields a sparse linear system $A \bm u = \bm b$. In our setting, $A$ is large and sparse, and we solve it using the preconditioned conjugate gradient (PCG) method. PCG is an iterative Krylov-subspace solver designed for symmetric positive definite (SPD) systems, and typically achieves fast convergence when combined with an effective preconditioner.

We employ a Jacobi (diagonal) preconditioner $M = D$, where $D=\text{diag}(A)$ is the diagonal of $A$. Given an initial guess $\bm u^0$, initialize
\begin{equation}
    \bm r^0 = \bm b - A \bm u^0, \ \bm z^0 = M^{-1} \bm r^0, \ \bm p^0 = \bm z^0.
\end{equation}

For step $k=0, 1, 2, \dots$, PCG performs step size
\begin{equation}
    \alpha^k = \frac{\bm r^k \cdot \bm z^k}{\bm p^k \cdot A\bm p^k},
\end{equation}
update the solution and residual by
\begin{equation}
    \bm u^{k+1} = \bm u^{k} + \alpha^k \bm p^k, \ \bm r^{k+1} = \bm r^{k} - \alpha^k A \bm p^k,
\end{equation}
then apply the preconditioner and update the direction through
\begin{equation}
    \bm {z}^{k+1} = M^{-1} \bm r^{k+1}, \ \beta^k = \frac{\bm r^{k+1} \cdot \bm z^{k+1}}{\bm r^k \cdot \bm z^k}, \ \bm p^{k+1} = \bm z^{k+1} + \beta^k \bm p^k
\end{equation}
We terminate PCG when the absolute residual norm satisfies $\Vert \bm r^k \Vert_2 \le \epsilon$ or when a maximum iteration budget is reached. In our numerical reference solver, PCG is executed until convergence (with $\epsilon = 1.0\times10^{-8}$) or until a maximum budget of 3000 iterations to produce high-quality solutions. In our iterative physics supervision, we instead run a fixed small number of PCG steps $K$ initialized from the network prediction and use the resulting iterate as a physics-consistent supervision target.

\section{Proofs for Section~\ref{sec:theory}}
\label{app:theory}

\subsection{Preliminaries}

We work in exact arithmetic. Let $F_K(\bm u)$ denote the output of applying $K$ steps of
(preconditioned) conjugate gradients (PCG) to $\bm A\bm u=\bm b$ starting from the initial guess
$\bm u_0=\bm u$, using an SPD preconditioner $\bm M$.
By convention, if the current residual becomes exactly zero at some step, PCG performs no further
updates and the iterate remains unchanged thereafter.

Define the self-consistency residual $\bm s_K(\bm u):=\bm u-F_K(\bm u)$ and the exact solution
$\bm u^\star=\bm A^{-1}\bm b$ (A1). For SPD $\bm A$, define $\|\bm x\|_{\bm A}=\sqrt{\bm x^\top\bm A\bm x}$.
Let $\bm C=\bm M^{-1/2}\bm A\bm M^{-1/2}$ and $\kappa=\kappa(\bm C)$, and define
$\rho=(\kappa-1)/(\kappa+1)\in[0,1)$.

\begin{lemma}[Norm equivalence under symmetric preconditioning]
\label{lem:norm_equiv}
Let $\bm y=\bm M^{1/2}\bm u$ and $\bm y^\star=\bm M^{1/2}\bm u^\star$. Then $\bm y^\star$ solves
$\bm C\bm y=\bm M^{-1/2}\bm b$, and for any $\bm u$,
\[
\|\bm u-\bm u^\star\|_{\bm A}=\|\bm y-\bm y^\star\|_{\bm C}.
\]
\end{lemma}
\begin{proof}
We have $\bm A\bm u=\bm b \iff \bm M^{-1/2}\bm A\bm M^{-1/2}\bm y=\bm M^{-1/2}\bm b$, i.e.,
$\bm C\bm y=\bm M^{-1/2}\bm b$.
Moreover,
\[
\begin{aligned}
\|\bm u-\bm u^\star\|_{\bm A}^2 &= (\bm u-\bm u^\star)^\top\bm A(\bm u-\bm u^\star) \\
&= (\bm y-\bm y^\star)^\top \bm M^{-1/2}\bm A\bm M^{-1/2}(\bm y-\bm y^\star) \\
&= \|\bm y-\bm y^\star\|_{\bm C}^2.
\end{aligned}
\]

\end{proof}

\begin{lemma}[Kantorovich inequality]
\label{lem:kantorovich}
Let $\bm C$ be SPD with eigenvalues in $[\lambda_{\min},\lambda_{\max}]$ and $\kappa=\lambda_{\max}/\lambda_{\min}$.
Then for any nonzero $\bm x$,
\[
(\bm x^\top \bm C\bm x)\,(\bm x^\top \bm C^{-1}\bm x)
\le \frac{(\lambda_{\max}+\lambda_{\min})^2}{4\lambda_{\max}\lambda_{\min}}\,(\bm x^\top\bm x)^2
= \frac{(\kappa+1)^2}{4\kappa}\,(\bm x^\top\bm x)^2.
\]
\end{lemma}
\begin{proof}
Diagonalize $\bm C=\bm Q\Lambda\bm Q^\top$ and let $\bm y=\bm Q^\top\bm x$.
Write $S=\sum_i y_i^2=\bm x^\top\bm x$ and weights $w_i=y_i^2/S$ so that $\sum_i w_i=1$.
Then
\[
\bm x^\top\bm C\bm x = S\sum_i w_i \lambda_i,\qquad
\bm x^\top\bm C^{-1}\bm x = S\sum_i w_i \lambda_i^{-1}.
\]
For each $\lambda_i\in[\lambda_{\min},\lambda_{\max}]$ we have
$(\lambda_i-\lambda_{\min})(\lambda_i-\lambda_{\max})\le 0$, i.e.
$\lambda_i^2-(\lambda_{\min}+\lambda_{\max})\lambda_i+\lambda_{\min}\lambda_{\max}\le 0$.
Dividing by $\lambda_i>0$ gives
\[
\lambda_i + \frac{\lambda_{\min}\lambda_{\max}}{\lambda_i}\le \lambda_{\min}+\lambda_{\max}.
\]
Averaging w.r.t.\ $w_i$ yields
\[
\sum_i w_i\lambda_i + \lambda_{\min}\lambda_{\max}\sum_i w_i\lambda_i^{-1}\le \lambda_{\min}+\lambda_{\max}.
\]
By AM--GM inequality,
\[
\begin{aligned}
2\sqrt{\lambda_{\min}\lambda_{\max}
\Big(\sum_i w_i\lambda_i\Big)\Big(\sum_i w_i\lambda_i^{-1}\Big)}
&\le
\sum_i w_i\lambda_i
+\lambda_{\min}\lambda_{\max}\sum_i w_i\lambda_i^{-1} \\
&\le \lambda_{\min}+\lambda_{\max}.
\end{aligned}
\]

Squaring and rearranging proves
\[
\Big(\sum_i w_i\lambda_i\Big)\Big(\sum_i w_i\lambda_i^{-1}\Big)
\le \frac{(\lambda_{\min}+\lambda_{\max})^2}{4\lambda_{\min}\lambda_{\max}}.
\]
Multiplying by $S^2$ completes the proof.
\end{proof}

\begin{lemma}[One-step steepest descent contraction]
\label{lem:sd_contraction}
Consider solving $\bm C\bm y=\bm d$ with SPD $\bm C$ by steepest descent with exact line search:
\[
\bm y^+ = \bm y + \alpha \bm r,\qquad \bm r=\bm d-\bm C\bm y,\qquad
\alpha=\frac{\bm r^\top\bm r}{\bm r^\top\bm C\bm r}.
\]
Let $\bm e=\bm y-\bm y^\star$ where $\bm y^\star=\bm C^{-1}\bm d$.
Then
\[
\|\bm e^+\|_{\bm C}\le \rho\,\|\bm e\|_{\bm C},\qquad
\rho=\frac{\kappa(\bm C)-1}{\kappa(\bm C)+1}.
\]
\end{lemma}
\begin{proof}
Since $\bm r=\bm d-\bm C\bm y= -\bm C(\bm y-\bm y^\star)=-\bm C\bm e$,
the update gives $\bm e^+=\bm e-\alpha\bm C\bm e$.
Expanding the $\bm C$-norm,
\[
\|\bm e^+\|_{\bm C}^2
=(\bm e-\alpha\bm C\bm e)^\top\bm C(\bm e-\alpha\bm C\bm e)
=\|\bm e\|_{\bm C}^2 - \frac{(\bm r^\top\bm r)^2}{\bm r^\top\bm C\bm r}.
\]
Also note $\bm r^\top\bm C^{-1}\bm r = \bm e^\top\bm C\bm e=\|\bm e\|_{\bm C}^2$.
Therefore
\[
\frac{\|\bm e^+\|_{\bm C}^2}{\|\bm e\|_{\bm C}^2}
= 1 - \frac{(\bm r^\top\bm r)^2}{(\bm r^\top\bm C\bm r)(\bm r^\top\bm C^{-1}\bm r)}.
\]
By Lemma~\ref{lem:kantorovich},
\[
\begin{aligned}
(\bm r^\top\bm C\bm r)(\bm r^\top\bm C^{-1}\bm r)
&\le \frac{(\kappa+1)^2}{4\kappa}(\bm r^\top\bm r)^2 \\
\Longrightarrow\qquad
\frac{(\bm r^\top\bm r)^2}{(\bm r^\top\bm C\bm r)(\bm r^\top\bm C^{-1}\bm r)}
&\ge \frac{4\kappa}{(\kappa+1)^2}.
\end{aligned}
\]

Hence
\[
\frac{\|\bm e^+\|_{\bm C}^2}{\|\bm e\|_{\bm C}^2}
\le 1-\frac{4\kappa}{(\kappa+1)^2}
=\Big(\frac{\kappa-1}{\kappa+1}\Big)^2=\rho^2,
\]
which implies $\|\bm e^+\|_{\bm C}\le \rho\|\bm e\|_{\bm C}$.
\end{proof}

\begin{lemma}[PCG-$K$ contraction in the energy norm]
\label{lem:pcg_contraction}
Assume A1--A2. Then for any initial guess $\bm u$,
\[
\|F_K(\bm u)-\bm u^\star\|_{\bm A}\le \rho^K\,\|\bm u-\bm u^\star\|_{\bm A}.
\]
\end{lemma}
\begin{proof}
Under SPD $\bm M$, standard symmetrically-preconditioned CG (PCG) on $(\bm A,\bm M)$ is equivalent
to applying (unpreconditioned) CG to the SPD system $\bm C\bm y=\bm M^{-1/2}\bm b$ in the variable
$\bm y=\bm M^{1/2}\bm u$.
CG produces the $\bm C$-norm optimal error over the $K$-step Krylov subspace; in particular, its
error after $K$ iterations is no worse than that of $K$ steps of steepest descent with exact line search.
By Lemma~\ref{lem:sd_contraction}, steepest descent contracts the $\bm C$-norm error by at most $\rho$
per step, hence after $K$ steps:
\[
\|\bm y_K-\bm y^\star\|_{\bm C}\le \rho^K\|\bm y_0-\bm y^\star\|_{\bm C}.
\]
Finally apply Lemma~\ref{lem:norm_equiv} to translate back to $\bm A$-energy norms:
$\|\bm y_K-\bm y^\star\|_{\bm C}=\|F_K(\bm u)-\bm u^\star\|_{\bm A}$ and
$\|\bm y_0-\bm y^\star\|_{\bm C}=\|\bm u-\bm u^\star\|_{\bm A}$.
\end{proof}

\subsection{Proof of Fixed-point Consistency}

\begin{theorem}[Fixed-point consistency]
\label{thm:fixed_point}
Assume A1--A2 and $K\ge 1$. Then for any $\bm u\in\mathbb R^N$,
\[
F_K(\bm u)=\bm u\quad\Longleftrightarrow\quad \bm A\bm u=\bm b.
\]
\end{theorem}
\begin{proof}
($\Rightarrow$) Suppose $\bm A\bm u=\bm b$. Then the initial residual is $\bm r_0=\bm b-\bm A\bm u=\bm 0$.
By the definition of PCG and our convention, no further updates are performed, hence $F_K(\bm u)=\bm u$.

($\Leftarrow$) Suppose $\bm A\bm u\neq \bm b$, i.e.\ $\bm r_0=\bm b-\bm A\bm u\neq \bm 0$.
In PCG, the first preconditioned residual is $\bm z_0=\bm M^{-1}\bm r_0\neq \bm 0$
(because $\bm M$ is nonsingular SPD), and the first search direction is $\bm p_0=\bm z_0$.
The first step size is
\[
\alpha_0=\frac{\bm r_0^\top\bm z_0}{\bm p_0^\top\bm A\bm p_0}.
\]
Since $\bm M$ and $\bm A$ are SPD, $\bm r_0^\top\bm z_0=\bm r_0^\top\bm M^{-1}\bm r_0>0$
and $\bm p_0^\top\bm A\bm p_0>0$, hence $\alpha_0>0$.
Therefore, the first iterate updates as $\bm u_1=\bm u+\alpha_0\bm p_0\neq \bm u$.
When $K\ge 1$, we have
$\bm u_K=\bm u+\sum_{i=0}^{K-1}\alpha_i\bm p_i$.
If $\bm u_K=\bm u$, then $\sum_{i=0}^{K-1}\alpha_i\bm p_i=\bm 0$.
Left-multiplying by $\bm p_0^\top\bm A$ and using the $A$-conjugacy
$\bm p_0^\top\bm A\bm p_i=0$ for $i\ge 1$, we obtain
$0=\alpha_0\,\bm p_0^\top\bm A\bm p_0$, contradicting $\alpha_0>0$ and SPD of $\bm A$.
Hence $\bm u_K\neq \bm u$. Therefore, $F_K(\bm u)=\bm u$ implies $\bm r_0=\bm 0$, i.e.\ $\bm A\bm u=\bm b$.
\end{proof}

\subsection{Proof of Stop-gradient Contraction}

\begin{theorem}[Stop-gradient contraction]
\label{thm:stop_grad}
Assume A1--A2 and $K\ge 1$. Consider the iteration
$\bm u_{t+1}=\bm u_t-\eta\,\bm s_K(\bm u_t)=(1-\eta)\bm u_t+\eta F_K(\bm u_t)$ with $\eta\in(0,1]$.
Let $\bm e_t=\bm u_t-\bm u^\star$. Then
\[
\|\bm e_{t+1}\|_{\bm A}\le \bigl((1-\eta)+\eta\rho^K\bigr)\,\|\bm e_t\|_{\bm A}
= \bigl(1-\eta(1-\rho^K)\bigr)\,\|\bm e_t\|_{\bm A}.
\]
\end{theorem}
\begin{proof}
Using $\bm u^\star=F_K(\bm u^\star)$ and the update definition,
\[
\begin{aligned}
\bm e_{t+1}=\bm u_{t+1}-\bm u^\star
&=(1-\eta)(\bm u_t-\bm u^\star)+\eta\big(F_K(\bm u_t)-\bm u^\star\big) \\
&=(1-\eta)\bm e_t+\eta\big(F_K(\bm u_t)-\bm u^\star\big).
\end{aligned}
\]

Taking $\bm A$-norms and applying the triangle inequality gives
\[
\|\bm e_{t+1}\|_{\bm A}\le (1-\eta)\|\bm e_t\|_{\bm A}+\eta\|F_K(\bm u_t)-\bm u^\star\|_{\bm A}.
\]
Finally apply Lemma~\ref{lem:pcg_contraction} to bound
$\|F_K(\bm u_t)-\bm u^\star\|_{\bm A}\le \rho^K\|\bm e_t\|_{\bm A}$, yielding the claim.
\end{proof}

\subsection{Proof of the Error-proxy Theorem}

\begin{theorem}[Self-consistency residual as an error proxy]
\label{thm:error_proxy}
Assume A1--A2 and $K\ge 1$. Then for any $\bm u$,
\[
(1-\rho^K)\,\|\bm u-\bm u^\star\|_{\bm A}
\le \|\bm s_K(\bm u)\|_{\bm A}
\le (1+\rho^K)\,\|\bm u-\bm u^\star\|_{\bm A}.
\]
Furthermore, letting $\bm r(\bm u)=\bm A\bm u-\bm b$, we have
\[
\lambda_{\min}(\bm A)\,\|\bm u-\bm u^\star\|_{\bm A}
\le \|\bm r(\bm u)\|_{\bm A}
\le \lambda_{\max}(\bm A)\,\|\bm u-\bm u^\star\|_{\bm A}.
\]
\end{theorem}
\begin{proof}
Let $\bm e=\bm u-\bm u^\star$ and $\bm e_K=F_K(\bm u)-\bm u^\star$. Then
$\bm s_K(\bm u)=\bm u-F_K(\bm u)=\bm e-\bm e_K$.
By the triangle inequality and Lemma~\ref{lem:pcg_contraction},
\[
\|\bm s_K(\bm u)\|_{\bm A}\le \|\bm e\|_{\bm A}+\|\bm e_K\|_{\bm A}
\le (1+\rho^K)\|\bm e\|_{\bm A}.
\]
Similarly, by the reverse triangle inequality,
\[
\|\bm s_K(\bm u)\|_{\bm A}\ge \bigl|\|\bm e\|_{\bm A}-\|\bm e_K\|_{\bm A}\bigr|
\ge (1-\rho^K)\|\bm e\|_{\bm A}.
\]
This proves the first inequality.

For the residual bound, note $\bm r(\bm u)=\bm A\bm u-\bm b=\bm A(\bm u-\bm u^\star)=\bm A\bm e$.
Then
\[
\|\bm r(\bm u)\|_{\bm A}^2
=\bm r^\top\bm A\bm r
=(\bm A\bm e)^\top\bm A(\bm A\bm e)
=\bm e^\top \bm A^3 \bm e.
\]
Let the eigenvalues of $\bm A$ lie in $[\lambda_{\min}(\bm A),\lambda_{\max}(\bm A)]$.
In an eigenbasis of $\bm A$, each component satisfies
$\lambda_{\min}(\bm A)^2\,\lambda_i \le \lambda_i^3 \le \lambda_{\max}(\bm A)^2\,\lambda_i$,
which implies the matrix inequality
$\lambda_{\min}(\bm A)^2\,\bm A \preceq \bm A^3 \preceq \lambda_{\max}(\bm A)^2\,\bm A$.
Therefore,
\[
\lambda_{\min}(\bm A)^2\,\bm e^\top\bm A\bm e
\le \bm e^\top\bm A^3\bm e
\le \lambda_{\max}(\bm A)^2\,\bm e^\top\bm A\bm e.
\]
Taking square roots yields
\[
\lambda_{\min}(\bm A)\,\|\bm e\|_{\bm A}
\le \|\bm r(\bm u)\|_{\bm A}
\le \lambda_{\max}(\bm A)\,\|\bm e\|_{\bm A}.
\]
\end{proof}

\subsection{Derivation for the Local Expansions in Remark}
\label{app:stopgrad_remark}

Assume $F_K$ is Fr\'echet differentiable at $\bm u^\star$ with Jacobian $\bm J=\nabla F_K(\bm u^\star)$.
For $\bm u=\bm u^\star+\bm e$ and small $\bm e$,
\[
\begin{aligned}
F_K(\bm u)
&=F_K(\bm u^\star+\bm e)
=\bm u^\star+\bm J\bm e + O\!\left(\|\bm e\|_2^2\right), \\
\bm s_K(\bm u)
&=\bm u-F_K(\bm u)
=(\bm I-\bm J)\bm e + O\!\left(\|\bm e\|_2^2\right).
\end{aligned}
\]

For the stop-gradient update $\bm u^+=\bm u-\eta\,\bm s_K(\bm u)$,
\[
\bm e^+=\bm u^+-\bm u^\star
=\bm e-\eta(\bm I-\bm J)\bm e+O(\|\bm e\|_2^2)
=\bigl((1-\eta)\bm I+\eta\bm J\bigr)\bm e+O(\|\bm e\|_2^2).
\]
For the full-gradient update on $\phi(\bm u)=\tfrac12\|\bm s_K(\bm u)\|_2^2$,
$\nabla\phi(\bm u)=(\nabla \bm s_K(\bm u))^\top \bm s_K(\bm u)$ and
$\nabla \bm s_K(\bm u^\star)=\bm I-\bm J$, hence
\[
\begin{aligned}
\nabla\phi(\bm u)
&=(\bm I-\bm J)^\top(\bm I-\bm J)\bm e
+ O\!\left(\|\bm e\|_2^2\right), \\
\bm e^+
&=\Bigl(\bm I-\eta(\bm I-\bm J)^\top(\bm I-\bm J)\Bigr)\bm e
+ O\!\left(\|\bm e\|_2^2\right).
\end{aligned}
\]

\section{More Details in Experiments}
\label{app_sec:exp}
The hyperparameters about the number of model parameters, epochs, batch size, and learning rate for training \textsc{NPSolver} and baselines across different tasks are summarized in Table~\ref{app_tab:super_params}. The default setting of PCG steps $K$ for \textsc{NPSolver} is 40. All source code and data will be made available after peer review.

\begin{table*}[!t]
    \centering
    \caption{The number of model parameters, epochs, batch size, and learning rate for training \textsc{NPSolver} across different tasks.}
    \begin{tabular}{cccccccc}
    \toprule
         Task  & Method & Params. & Epochs & Batch size & LR \\
         \midrule
         \multirow{4}{*}{2D Dirichlet} & \textsc{PI-DeepONet} & 3.6M & 400 & 100 & 0.001 \\
         & \textsc{PINO} & 3.5M & 400 & 8 & 0.001 \\
         & \textsc{PINN} & 0.2M & 11000 & 1 & 0.001 \\
          & \textsc{NPSolver} & 3.6M &  400 & 8 & 0.001 \\
          \midrule
         \multirow{4}{*}{2D Neumann} & \textsc{PI-DeepONet} & 3.6M & 600 & 100 & 0.0005 \\
         & \textsc{PINO} & 3.5M & 600 & 8 & 0.0005 \\
         & \textsc{PINN} & 0.2M & 11000 & 1 & 0.0005 \\
          & \textsc{NPSolver}& 3.6M & 600 & 8 & 0.0005 \\
          \midrule
         \multirow{7}{*}{2D RandomBC} & \textsc{Transolver}& 3.5M &  1000 & 8 & 0.0005 \\
         & \textsc{Transolver++}& 2.9M &  1000 & 8 & 0.0005 \\
         & \textsc{MGN}& 2.2M &  1000 & 4 & 0.0005 \\
         & \textsc{GPS}& 2.6M &  1000 & 4 & 0.0005 \\
         & \textsc{PointNet++}& 3.5M &  1000 & 8 & 0.0005 \\
         & {\textsc{BENO}} & 3.4M & 1000 & 6 & 0.0005 \\
          & \textsc{NPSolver}& 3.6M &  1000 & 8 & 0.0005 \\
        \midrule
         \multirow{2}{*}{3D}
         & {\textsc{Transolver}} & 4.7M & 400 & 1 & 0.0005\\
         & \textsc{NPSolver}& 4.9M & 400 & 1 & 0.0005\\
         \midrule
         Control & \textsc{NPSolver}& 3.6M &  1000 & 8 & 0.0005\\
         \bottomrule
    \end{tabular}
    \label{app_tab:super_params}
\end{table*}
\subsection{2D Generalization under Different BCs}
\label{app_sec:2d}

\begin{figure}
    \centering
    \includegraphics[width=\linewidth]{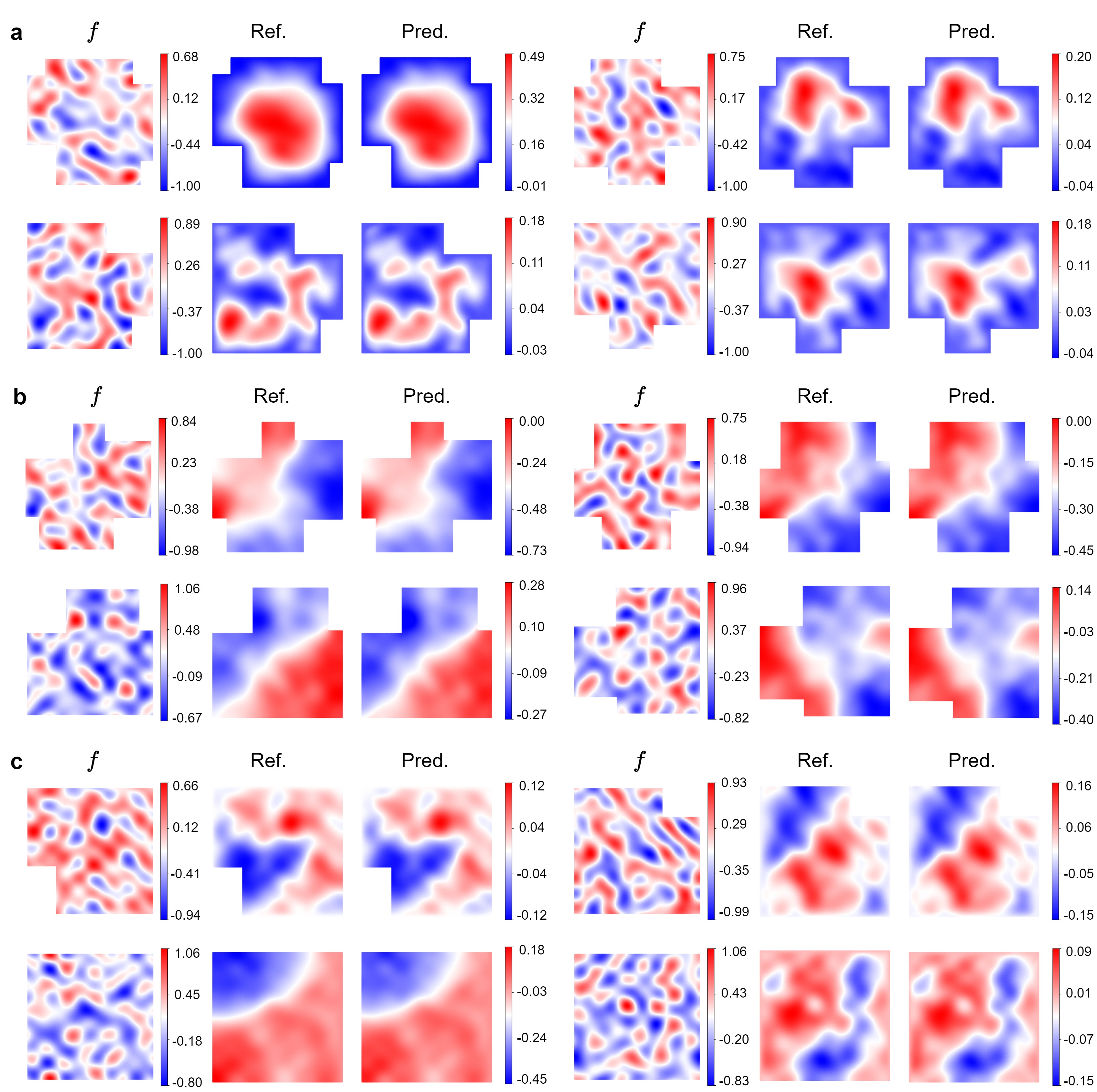}
    \caption{Visualization of \textsc{NPSolver}'s predictions for representative samples. (a) All-Dirichlet case (C4 and C2). (b) All-Neumann case (C4 and C2). (c) RandomBC case (C1 and C0).}
    \label{app_fig:2d}
\end{figure}

\hspace*{\parindent}\textbf{Geometry.} Following BENO~\cite{wang2024beno}, we construct a corner-removed square family with five categories, denoted as {C$k$}, where $k \in \{0, \dots, 4\}$ is the number of removed rectangular corners (C0 corresponds to the intact square). The base domain is defined on $[0, L]^2$ with $L=2\pi$ and a base resolution $128\times128$. For each domain geometry, the width and height of the removed corner rectangles are sampled independently from a uniform distribution over $[0.1L, 0.4L]$. To evaluate both in-distribution and out-of-distribution generalization ability, we \emph{train} the model only on the most irregular category C4 and \emph{test} it on all categories C0-C4.  We prepare 100 domain instances forming the training geometry set and 10 domain instances per category forming the testing geometry set. The average number of cells per category is summarized in Table~\ref{app_tab:2d_cells}.

\textbf{Forcing distribution.}
Given 3D spatial coordinates $\{\bm x_n\}_{n=1}^{N}$ with $\bm x_n=(x_n,y_n,z_n)$ at cell centroids of control volumes $\{ V_n \}^N_{n=1}$, where $z_n = 0$ for 2D, define the frequency index set
\[
\mathcal{K}=\left\{(i,j,k)\ \big|\ i,j,k\in\{0,1,\dots,M-1\}\right\},
\]
and let $A_{ijk},B_{ijk}\overset{\text{i.i.d.}}{\sim}\mathcal{N}(0,1)$ for $(i,j,k)\in\mathcal{K}$, and $c\sim\mathcal{N}(-1,1)$.
For each centroid $\bm x_n$, define the phase
\[
\phi_{ijk}(\bm x_n) = \left(i-\left\lfloor\frac{M}{2}\right\rfloor\right)x_n
+\left(j-\left\lfloor\frac{M}{2}\right\rfloor\right)y_n
+\left(k-\left\lfloor\frac{M}{2}\right\rfloor\right)z_n,
\]
and construct
\[
W(\bm x_n)=\sum_{i=0}^{M-1}\sum_{j=0}^{M-1}\sum_{k=0}^{M-1}
\Big( A_{ijk}\sin\!\big(\phi_{ijk}(\bm x_n)\big)+B_{ijk}\cos\!\big(\phi_{ijk}(\bm x_n)\big)\Big).
\]
Add the bias term:
\[
f_0(\bm x_n)=W(\bm x_n)+c.
\]
The returned field is
\[
f(\bm x_n)=\frac{f_0(\bm x_n)}{\max_{1\le m\le N}\big|f_0(\bm x_m)\big|+\varepsilon}.
\]
In practice, we choose $M=10$ to ensure sufficient spectral richness. For the pure zero-flux Neumann case, in order to satisfy the compatibility condition $\int_\Omega f = 0$, we normalize $f$ by
\[
\bar f \;=\; \frac{\sum_{n=1}^{N} f(\bm x_n)\, \vert V_n \vert}{\sum_{n=1}^{N} \vert V_n \vert},
\qquad
f(\bm x_n)^{\mathrm{norm}} \;=\; f(\bm x_n) - \bar f,
\]
where $\vert V_i \vert$ is the volume for each control volume $V_i$.

For evaluation, we pair 50 test domain instances (10 per category) with 10 independent forcing realizations each. This yields 100 test samples per category and 500 total test samples per case, with ground-truth solutions provided by the numerical solver.

\textbf{Implementation details of baselines.} To ensure a fair and comprehensive comparison, we adapt three representative physics-informed baselines to our irregular domain and variable BC settings. All models, including \textsc{NPSolver}, are trained using identical optimization protocols: the Adam optimizer with a OneCycle learning-rate schedule for an equal number of epochs.
\begin{itemize}[leftmargin=12pt]
    \item \textsc{PI-DeepONet}: The standard branch net architecture implicitly assumes a fixed sensor configuration, making it incompatible with irregular domains of varying node counts. To address this, we interpolate the irregular input fields onto a fixed $128 \times 128$ Cartesian grid to serve as the branch net input. The trunk net then queries coordinates from the original irregular mesh.
    \item \textsc{PINO}: Since the FNO backbone requires regular grids, we perform bi-directional interpolation and masking between the irregular physical domain and a latent regular grid for the network's input and output. Crucially, the finite-volume equation residuals are still computed on the original irregular mesh to maintain physical consistency during training.
    \item \textsc{PINN}: Unlike the models above, \textsc{PINN} is optimized from scratch for each test sample. The optimization follows a two-stage strategy: an initial 1,000 Adam iterations for global search, followed by 10,000 L-BFGS iterations to ensure fine-scale convergence. While providing a high-accuracy reference, this process incurs a prohibitive computational cost during inference. We evaluate \textsc{PINN} on a representative subset by selecting three samples from each category. The average performance across these selections is reported as the final estimate for the category.
    \item Data-driven baselines (\textsc{Transolver}, \textsc{Transolver++}, \textsc{MGN}, \textsc{GPS}, \textsc{PointNet++}, {\textsc{BENO}}): These data-driven baselines are inherently suited for modeling boundary-value problems on irregular geometries. For a fair comparison, all models are trained for the same number of epochs as \textsc{NPSolver}. Due to GPU memory constraints, the batch size is set to 8 for \textsc{Transolver}, \textsc{Transolver++}, \textsc{PointNet++}, and \textsc{NPSolver}, 4 for \textsc{MGN} and \textsc{GPS}, {and 6 for \textsc{BENO}}. Notably, the number of iterations per epoch for \textsc{NPSolver} is fixed at 100, whereas for other models, it scales with the dataset size (e.g., 125 iterations per epoch for 1k samples with a batch size of 8). This configuration intentionally favors the data-driven baselines by granting them a higher frequency of parameter updates, further underscoring the efficiency of our approach.
\end{itemize}

\begin{table}[!t]
    \centering
    \caption{Average number of cells per category in the 2D computational domain.}
    \label{app_tab:2d_cells}
    \begin{tabular}{cccccc}
        \toprule
        & C4 & C3 & C2 & C1 & C0  \\
        \midrule
        Avg. Cells & 12567 & 13806 & 14678 & 15245 & 16384 \\   
        \bottomrule
    \end{tabular}
\end{table}

\subsection{3D Poisson on Cube-with-Cylindrical-Hole}
\label{app_sec:3d}
\hspace*{\parindent}\textbf{Geometry and BCs.} We construct a 3D domain family by subtracting a cylindrical cavity from a cube $[0, L]^3$ with $L=2\pi$, discretized at a base resolution of $48^3$. The cylinder's diameter $d$ is sampled uniformly from $[0.1L,0.3L]$, and its center is randomized such that the cylinder remains fully contained within the cube. We prepare 40 domain instances forming the training geometry set and 10 domain instances forming the testing geometry set. The average number of cells of the testing geometry set is about 100,420. Regarding BCs, we impose zero-flux Neumann BCs on the outer cube surfaces and Dirichlet BCs with zero value on the internal cylindrical surface. This mixed-BC configuration couples an internal Dirichlet surface with an outer Neumann boundary, creating a challenging solution landscape.

\textbf{Forcing distribution.} The distribution of the forcing field in 3D is identical to that used in 2D.

For evaluation, we pair 10 test domain instances with 10 independent forcing realizations each, yielding 100 test samples with ground-truth solutions provided by the numerical solver.

\subsection{Thermal Control on Perforated Plate}
\label{app_sec:control}
\begin{figure}[!t]
    \centering
    \includegraphics[width=0.7\linewidth]{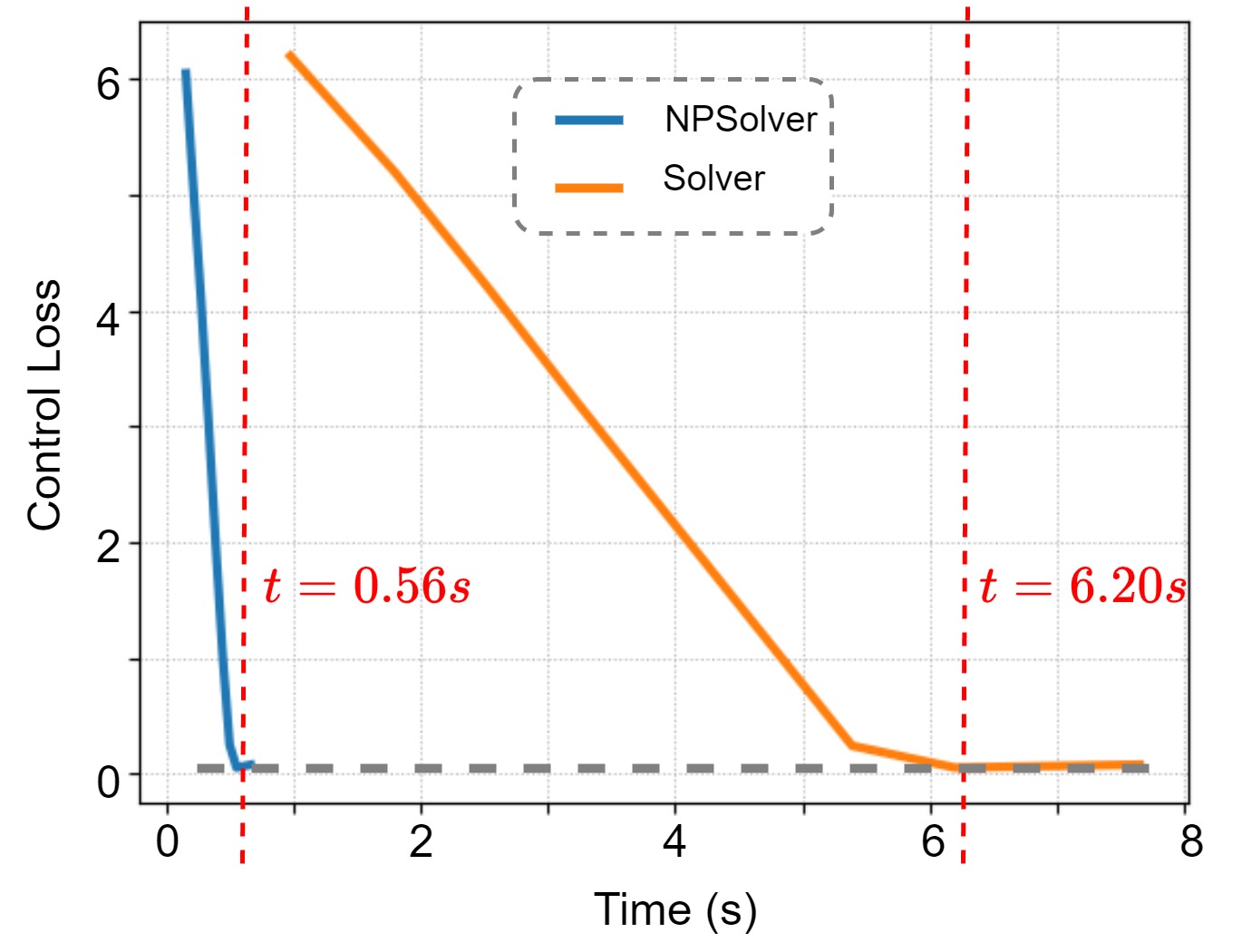}
    \caption{Control-loss convergence curve for a representative sample in the thermal control task.}
    \label{fig:control_curve}
\end{figure}

\hspace*{\parindent}\textbf{Geometry and BCs.} We construct a family of 2D domains by subtracting a single rectangular cavity from a square domain $[0, L]^2$ with $L=2\pi$, discretized at a base resolution of $128^2$. For each domain geometry, the cavity's width and height are sampled independently from $[0.1L,0.4L]$, with its center sampled such that the hole remains fully contained within the square. We prepare 100 domain instances forming the training geometry set and 10 domain instances forming the testing geometry set. The average number of cells of the testing geometry set is about 15,080. Controllable cooling is applied through four independently controlled Dirichlet segments on the bottom boundary, with values parameterized by a vector $\bm c = (c_1, c_2, c_3, c_4)$. All other boundaries are assigned zero-flux Neumann BCs, modeling insulated surfaces.   

\textbf{Forcing distribution.} To simulate randomly distributed heating elements, the source term $f$ is sampled from the specially designed distribution, generated as a superposition of localized Gaussian "hot spots" with randomized centers, amplitudes, and spatial spreads. Given 2D spatial coordinates $\{ \bm x_i \}_{i=1}^n$ with $\bm x_i=(x_i,y_i)\in [0, L^2]$,
we generate a nonnegative heat-source field $f:\{\bm x_i\}\to\mathbb{R}_{\ge 0}$ as a sum of $K$ Gaussian hot spots:

1) Number of hot spots:
\[
K \sim \mathrm{Unif}\{K_{\min}, K_{\max}\}.
\]

2) Hot-spot centers:
we sample $K$ indices $\{c_k\}_{k=1}^K$ i.i.d. from $\{1,\dots,n\}$, and set
\[
\bm \mu_k = \bm x_{c_k} = (\mu_{k,x},\mu_{k,y}).
\]

3) Amplitudes and widths:
\[
A_k \sim \mathrm{Unif}(a_{\min}, a_{\max}), \quad s_k \sim \mathrm{Unif}(s_{\min}, s_{\max}),
\]
and we use isotropic widths
\[
\sigma_{k,x}=\sigma_{k,y}=s_k.
\]

4) Unnormalized field:
for each point $\bm x_i=(x_i,y_i)$,
\[
f_0(\bm x_i)
=
\sum_{k=1}^K
A_k \exp\!\left(
-\frac{1}{2}\left(\frac{x_i-\mu_{k,x}}{\sigma_{k,x}+\varepsilon}\right)^2
-\frac{1}{2}\left(\frac{y_i-\mu_{k,y}}{\sigma_{k,y}+\varepsilon}\right)^2
\right),
\quad \varepsilon=10^{-8}.
\]

5) Mean-power scaling:
define the sample mean
\[
\bar f_0 = \frac{1}{n}\sum_{i=1}^n f_0(\bm x_i).
\]
Then the returned heat-source field is scaled:
\[
f(\bm x_i)
=
f_0(\bm x_i)\cdot
\frac{1}{\bar f_0+\varepsilon}.
\]
In practice, we choose $K_{\min}=2, K_{\max}=6, a_{\min}=0.5, a_{\max}=2.0, s_{\min}=0.02L, s_{\max}=0.08L, L=2\pi$ to produce spatially sparse, inhomogeneous sources that resemble heat injection from discrete components.

For evaluation, we pair 10 test domain instances with 10 independent forcing and randomized Dirichlet values $\bm c$, yielding a total of 100 test samples with ground-truth solutions provided by the numerical solver.

\section{Additional Analysis and  Experiments}
\label{app_sec:analysis}

\begin{figure}
    \centering
    \includegraphics[width=0.7\linewidth]{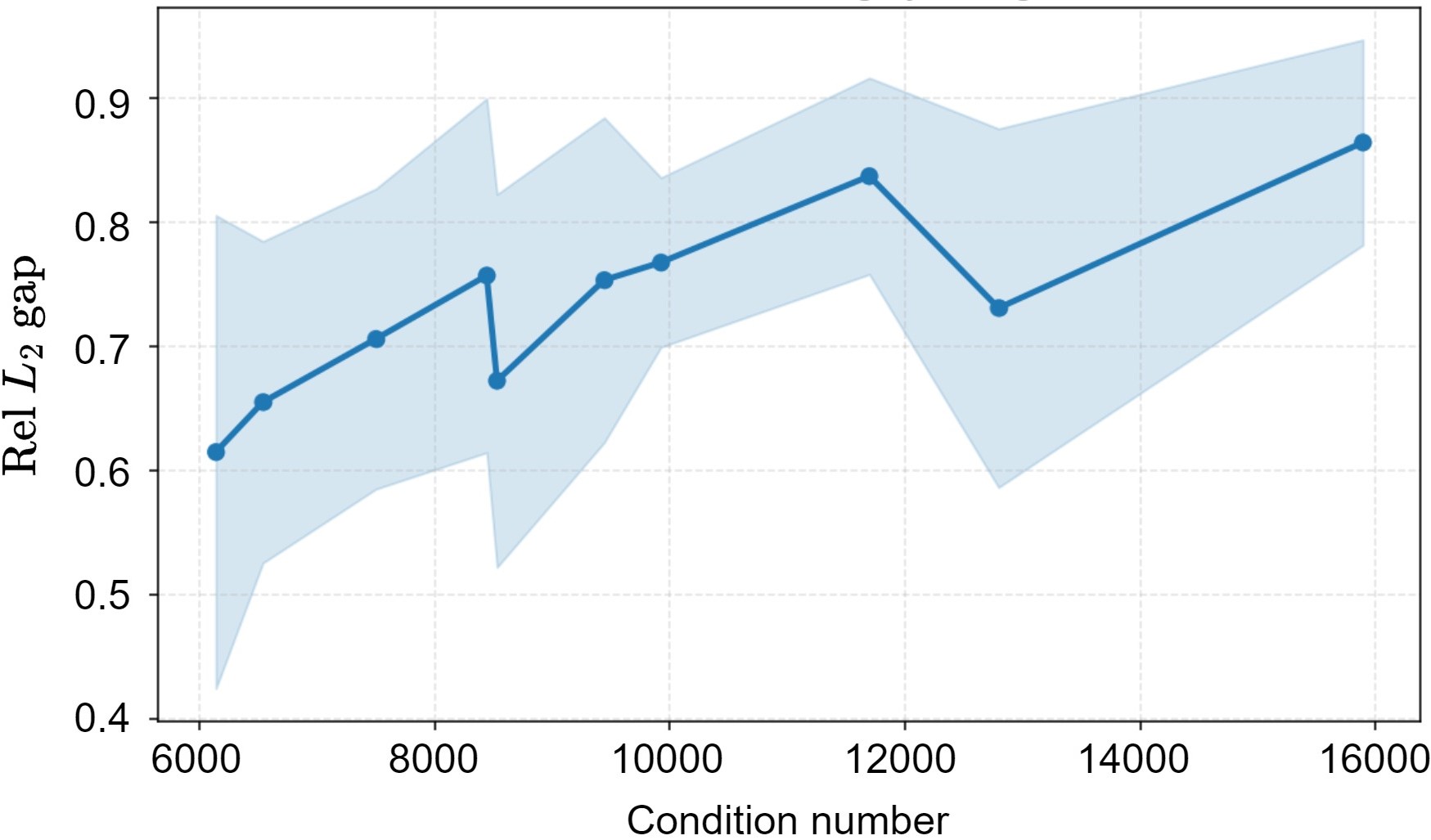}
    \caption{{Relationship between the estimated condition number and the relative $L_2$ error gap between residual supervision and iterative supervision on C4.}}
    \label{app_fig:cond_analysis}
\end{figure}
\textbf{Condition-number analysis.} We further analyze the relationship between the estimated condition number and the performance gap between residual supervision and iterative supervision. Since the model is trained on C4 and tested on C0–C4, cross-category comparisons mix conditioning and OOD geometry generalization. We therefore restrict this analysis to samples within the C4 family. As shown in Fig~\ref{app_fig:cond_analysis}, the relative $L_2$ gap shows an overall increasing trend with the estimated condition number. The Pearson correlation between the estimated condition number and the relative $L_2$ gap is 0.84, supporting our claim that residual supervision is more sensitive to ill-conditioning while iterative supervision provides a better-scaled signal.

\begin{table}[t]
    \centering
    \caption{{Relative $L_2$ errors (\%) on variable-coefficient Poisson equation in the All-Dirichlet BCs.}}
    \label{app_tab:variable_poisson}
    \begin{tabular}{cccccc}
         \toprule
         Method & C4 & C3 & C2 & C1 & C0  \\
         \midrule
         \textsc{PINO} & 25.18 & 24.66 & 29.09 & 30.73 & 32.88 \\
         \textsc{NPSolver} & \textbf{1.83} & \textbf{3.07} & \textbf{3.82} & \textbf{4.75} & \textbf{5.50} \\
         \bottomrule
    \end{tabular}
    \label{tab:placeholder}
\end{table}
\textbf{Variable-coefficient Poisson equation.} We add experiments on the variable-coefficient Poisson equation $\nabla \cdot (D(\bm x)\nabla u (\bm x))=f (\bm x)$ in the All-Dirichlet BCs setting. Following the setup in Section~\ref{sec:2d_exp}, the model is trained on the most irregular category C4 and tested on C0–C4, where it must simultaneously generalize across the geometry domain $\Omega$, the coefficient field $D(\bm x)$, and the forcing field $f(\bm x)$. As shown in Table~\ref{app_tab:variable_poisson}, \textsc{NPSolver} outperforms \textsc{PINO} by a large margin on all test categories, confirming that our framework also works well for variable-coefficient Poisson equations.

\begin{figure}
    \centering
    \includegraphics[width=\linewidth]{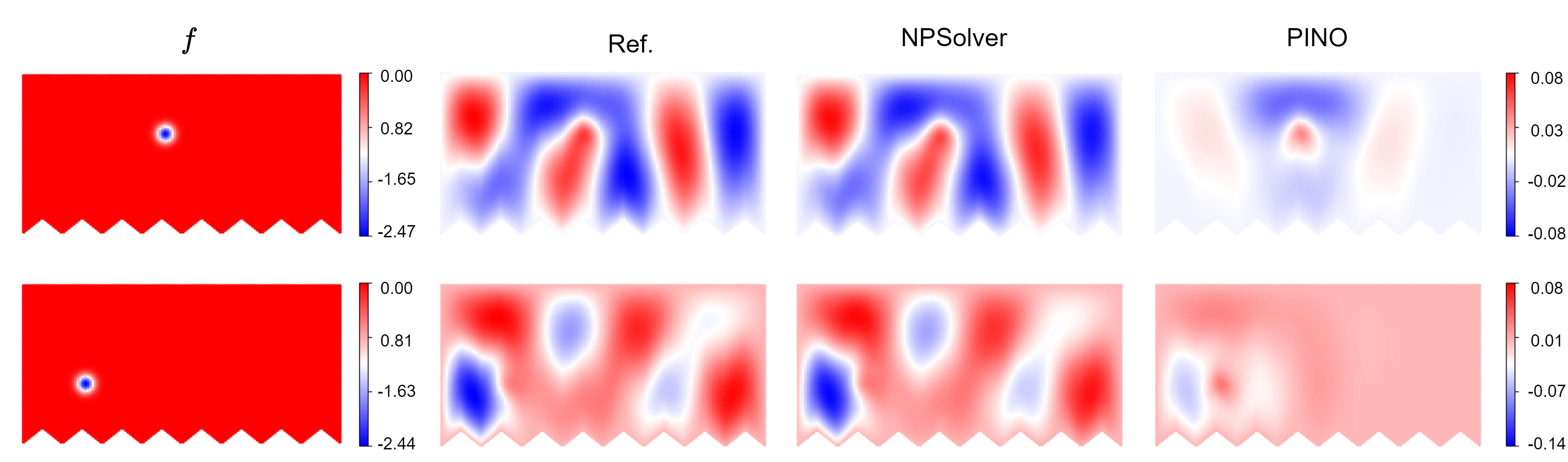}
    \caption{Visualization of models' predictions for representative samples.}
    \label{app_fig:2d_helm}
\end{figure}
\textbf{Helmholtz equation.} We further broaden the empirical evaluation with the 2D Helmholtz equation, following The Well Benchmark~\cite{ohana2024well}. Specifically, we consider $\Delta u+\omega^2u=-\delta _{x_0}$ on a rectangular region with a jagged lower boundary, whose bounding box is $[0, 2\pi]\times[0, \pi]$ and discretized with approximately $128\times64$ nodes and Dirichlet BCs. The task requires generalization across varying point-source right-hand sides $-\delta_{x_0}$. We train \textsc{NPSolver} in the same label-free manner via iterative physics supervision and compare it with \textsc{PINO} as a representative baseline. On this Helmholtz case, NPSolver achieves a relative $L_2$ error of 3.37\%, whereas \textsc{PINO} obtains 87.60\%. A representative prediction snapshot is shown in Fig.~\ref{app_fig:2d_helm}. This experiment provides additional evidence that the core iterative supervision idea is not limited to the exact Poisson problem considered in the main paper.

\end{document}